\documentclass{article}
\usepackage{algorithm}
\usepackage{algpseudocode}
\usepackage{mathtools}
\mathtoolsset{showonlyrefs}

\usepackage[preprint]{neurips_2025}
\usepackage{comment}
\usepackage{xcolor}
\usepackage{tikz}
\usepackage{mathtools}
\usetikzlibrary{arrows.meta, positioning, shapes.geometric}
\usepackage{enumitem}

\usepackage[utf8]{inputenc} %
\usepackage[T1]{fontenc}    %
\usepackage{hyperref}       %
\usepackage{url}            %
\usepackage{booktabs}       %
\usepackage{nicefrac}       %
\usepackage{microtype}      %
\usepackage{xcolor}         %
\usepackage{amsmath, amssymb, amsfonts, amsthm, yhmath}
\usepackage{bbm}
\usepackage{multirow}
\usepackage{tikz}
\usetikzlibrary{arrows.meta, positioning, shapes.geometric}
\usepackage{subcaption}
\newtheorem{theorem}{Theorem}
\newtheorem{assumption}{Assumption}
\newtheorem{proposition}{Proposition}
\newtheorem{lemma}{Lemma}

\title{Synthetic Survival Control: \\Extending Synthetic Controls for ``When-If'' Decision}

\author{%
  Jessy Xinyi Han \\
  Massachusetts Institute of Technology\\
  Cambridge, MA \\
  \texttt{xyhan@mit.edu} \\
  \And
  Devavrat Shah \\
  Massachusetts Institute of Technology\\
  Cambridge, MA \\
  \texttt{devavrat@mit.edu} \\
}

\begin{document}

\maketitle

\begin{abstract}
  Estimating causal effects on time-to-event outcomes from observational data is particularly challenging due to censoring, limited sample sizes, and non-random treatment assignment. The need for answering such ``when-if'' questions--how the \textit{timing} of an event would change under a specified intervention--commonly arises in real-world settings with heterogeneous treatment adoption and confounding. To address these challenges, we propose Synthetic Survival Control (SSC) to estimate counterfactual hazard trajectories in a panel data setting where multiple units experience potentially different treatments over multiple periods. 
  In such a setting, SSC estimates the counterfactual hazard trajectory for a unit of interest as a weighted combination of the observed trajectories from other units. To provide formal justification, we introduce a panel framework with a low-rank structure for causal survival analysis. Indeed, such a structure naturally arises under classical parametric survival models. Within this framework, for the causal estimand of interest, we establish identification and finite sample guarantees for SSC. We validate our approach using a multi-country clinical dataset of cancer treatment outcomes, where the staggered introduction of new therapies creates a quasi-experimental setting. Empirically, we find that access to novel 
  treatments is associated with improved survival, as reflected by lower post-intervention hazard trajectories relative to their synthetic counterparts. Given the broad relevance of survival analysis across medicine, economics, and public policy, our framework 
  offers a general and interpretable tool for counterfactual survival inference using observational data.
\end{abstract}

\section{Introduction}\label{sec:introduction}
Understanding and modeling the time until an event of interest happens is a fundamental problem across various domains, ranging from patient survival in healthcare \citep{han-petal-2024}, equipment failure in engineering \citep{fi15050153}, criminal recidivism in criminal justice \citep{han2025fairnessalgorithmsracialdisparities}, to customer churn \citep{gao2024causalcustomerchurnanalysis} and dynamic pricing in insurance \citep{pmlr-v202-choi23c}. Unlike classical machine learning algorithms focusing on binary or average outcomes, survival analysis explicitly models how risk evolves over time generally through a \emph{survival function}, which represents the probability of remaining event-free up to a certain time point, or a \emph{hazard (risk) function}, which describes the instantaneous probability of the event occurring at a given moment, conditional on not having experienced the event yet. Moreover, survival analysis accounts for censoring, where the event of interest is not observed within the study window or is preempted by a censoring event. As Samuel Beckett writes in \emph{Waiting for Godot}, ``Nothing happens, nobody comes, nobody goes, it's awful''–yet even the wait is informative. The absence of an event still revises our beliefs about underlying risk \citep{JMLR:v25:23-0888}. This time-to-event perspective enables us to understand not just \emph{whether} an event happens but also \emph{when} and more specifically, how the probability of occurrence evolves as time passes.

In applied settings, survival analysis plays more than the aforementioned descriptive role through estimation of survival distributions; it can also be used inferentially to assess how interventions affect the timing and risk of events. For example, do novel cancer therapies prolong patient survival? Our work is motivated precisely by such a real-world problem in oncology: evaluating the effectiveness of new therapies for patients with \emph{relapsed or refractory} T-cell lymphoma (TCL), a rare and aggressive form of blood cancer with historically poor prognosis \citep{han-petal-2024}. Typically, cancer patients first receive an initial, or \emph{first-line}, treatment. The disease is considered \emph{refractory} if this initial treatment fails to work at all and called \emph{relapsed} if it initially succeeds but the cancer later returns. In either scenario, patients then require an alternative, known as \emph{second-line}, treatment. Recently, novel single agent (SA) have been adopted as choices for second-line therapies in the United States, while many other countries continue to rely primarily on conventional chemotherapy (CC). This staggered adoption across countries provides a natural quasi-experimental setting for investigating the causal effects of novel therapies on patient survival. Such investigations require constructing \emph{counterfactual survival functions}, which are estimates of how patient survival trajectories would have evolved under alternative treatment strategies.

Inferring such counterfactual survival functions presents additional challenges beyond those faced in classical statistical analyses. First, survival analysis inherently deals with sparse data: individual-level right-censored data typically provides only a single observed time point—the minimum of the actual time-to-event and time-to-censoring, yet the goal is to recover full survival or hazard trajectories under different interventions. In the context of relapsed or refractory TCL, these challenges are further exacerbated. As a rare disease, TCL has limited participation from clinical centers, resulting in small sample sizes and magnified statistical uncertainty. Moreover, its aggressive progression makes it difficult to distinguish survival patterns between treatment groups over short follow-up periods, intensifying the need for more effective therapies and reliable estimation methods. Adding to these difficulties, the available data are observational and collected retrospectively across different countries and institutions. Treatment assignment is therefore non-random and occurs by country-level drug availability, making it particularly challenging to isolate causal effects from spurious associations.

To formalize this problem, it is helpful to situate it within a panel data setup. In classical panel data applications, the outcomes associated with a unit (e.g., country or institution) are repeated measurements of a scalar variable (a single real number) across multiple periods. Methods such as Synthetic Control \citep{abadie2010synthetic} and Synthetic Interventions \citep{agarwal2024syntheticinterventions} operate in such {\em data-rich} scenarios to leverage the structure among units for counterfactual estimations in the presence of confounding. However, our setup is {\em data-sparse} by the nature of time-to-event data: each individual often contributes only one observed time-to-event, or none if censored. Moreover, our target is not a scalar but an entire hazard function over time. This leads to the central question of our work: {\em Can methods like Synthetic Control, originally designed for scalar outcomes in data-rich scenarios, be generalized to handle functional, time-evolving survival processes under censoring in data-sparse setups?}

\subsection{Contributions}

In this work, our main contribution is to address this question by bridging the gap between the literature on Synthetic Control and Survival Analysis. 
We propose a panel data setting for survival analysis where multiple units are present and each unit experiences one of the possible treatments in each period. Within a unit and period, we observe multiple observations of, potentially censored, time-to-event. The potential outcomes associated with units in periods under a treatment can be characterized through ``potential survival functions'', which can be factorized into (latent) characteristics associated with units, periods, and treatments. This factorized representation is a natural extension of classical survival models such as Cox, where the log-hazard decomposes into additive contributions from baseline risk, covariates, and treatment. Generalizing this decomposition to higher dimensions suggests a low-rank or factorized  structure, with latent unit and period factors playing the role of unobserved confounders. Crucially, we allow for treatment assignment to arbitrarily depend on these (latent) factors of, thereby permitting (unobserved) confounding. 

Within this causal framework, we consider a causal estimand analogous to that in the traditional framework of synthetic control (see Section \ref{ssec:setup}). For such a causal estimand, we establish identification under minimal conditions (see Theorem \ref{theorem:theorem1}). To estimate the causal estimand, we propose the method of Synthetic Survival Control (SSC) (see Section \ref{ssec:ssc}). We establish finite sample consistency guarantee (and hence inference) property for such an estimator (see Theorem \ref{thm:consistency}). 

To demonstrate the soundness of the framework, %
we first validate it on synthetic data. 
We apply it to a real-world scenario in the context of evaluating outcomes associated with TCL cancer treatment. Specifically, through collaboration with leading cancer research institutes, we obtain clinical data across multiple countries (units) over different periods. For each patient, typically, there are multiple stages of treatment. We consider two stages (called \textit{first-line} and \textit{second-line}). Across all countries, patients receive the treatment CC (called \textit{control}) in stage one and two, with the exception of the United States where in the second stage, patients receive SA (called \textit{intervention}). In each stage, the outcome of interest is the time to treatment failure, defined as either the initiation of a new line of therapy or death. Therefore, this naturally forms a panel setup for survival analysis with two periods (first stage and second stage of treatment). As it happens, the dataset also contains patients in the US who receive CC in the second stage as well. We withhold such data and use SSC to do counterfactual estimation for patient outcomes in the US under CC in the second stage. We find that SSC faithfully recovers such counterfactual hazard function and thus validates the method. This is particularly significant given all the data challenges present in such observational clinical data.

Together, our work illustrates that counterfactual estimation on survival functions is feasible even with sparse, censored data when leveraging group-level quasi-experiments and survival-specific estimands. By integrating time-to-event modeling with causal inference, our work contributes to a growing literature on counterfactual survival analysis. %
In summary, this work offers a practical framework for survival-based counterfactual estimation, bridging methodological gaps and addressing real-world questions.

\subsection{Organization}

Section~\ref{sec:related} reviews related work on causal inference for survival outcomes and synthetic control methods. Section~\ref{sec:method} introduces our panel framework with low-rank structure, including identification, SSC estimator, and theoretical guarantees. Section~\ref{sec:empirics} describes the TCL dataset and and presents our empirical findings on treatment efficacy. Section~\ref{sec:conclusion} discusses implications, limitations, and future direction.

\section{Related Works}\label{sec:related}
\textbf{Synthetic Control, Panel Methods and Matrix Completion.}
Synthetic control methods reconstruct counterfactuals for a single treated unit by optimally weighting a pool of control units, under the assumption that a linear combination of controls can approximate the treated unit's pre‑treatment path. \citet{abadie2003economic} first applied this idea in economics to assess the costs of conflict in the Basque Country, and \citet{abadie2010synthetic} formalized it in a statistical framework for comparative case studies. Subsequent work has extended and unified synthetic control with other panel‑data approaches: \citet{doudchenko2016balancing} proposed a unified framework that integrates elements from synthetic controls, difference-in-differences (DiD), and regression-based approaches to address limitations inherent in each method when applied separately, and \citet{xu2017generalized} proposed a generalized synthetic control via interactive fixed effects models.  More recently, \citet{arkhangelsky2021synthetic} blended synthetic control with difference‑in‑differences in the synthetic DiD estimator, improving robustness to heterogeneous trends. \citet{agarwal2024syntheticinterventions} extended the synthetic control estimator to a synthetic intervention method which addresses the complementary question of what could have happened had a control unit received the treatment. Extending scalar outcomes to distributions, \citet{feitelberg2025distributionalmatrixcompletionnearest} developed a latent-factor nearest-neighbor method for imputing missing empirical distributions in a sparsely observed matrix by leveraging Wasserstein geometry, but do not model temporal structure, censoring, or interventions.

\textbf{Survival Analysis.}
Classical survival analysis focuses on modeling hazard and survival functions under censoring, with foundational works including Kaplan-Meier's non-parametric estimator \citep{kaplan1958nonparametric} and Cox's proportional hazards model \citep{cox1972regression}. These works typically reply on assumptions like random assignment of treatment, non-informative censoring, no noncompliance and so on. Along with the development of causal inference, these strong assumptions are gradually relaxed. \citet{robins1992correcting} developed rank‑preserving structural failure time models to correct for noncompliance; \citet{ipcw} proposed the inverse probability of censoring weighted (IPCW) method to address the missing data problem due to dependent censoring,  and marginal structural models in \citet{robins2000marginal} use inverse‑probability weighting to estimate time‑varying treatment effects.  

\textbf{Causal Survival Analysis.}
Several recent studies have explored how to adapt causal inference methods to time-to-event outcomes. One line of work focuses on adapting DiD to the survival setting \citep{wuHazardLinearProbability2022, wooldridgeSimpleApproachesNonlinear2023, deanerCausalDurationAnalysis2024}. \citet{cui2023} developed doubly robust and tree-based survival estimators. Other works take the perspective of matrix completion in past literature. For example, \citet{gao2024causalcustomerchurnanalysis} proposed a tensor block hazard model using 1-bit tensor completion for binary churn outcomes and intervention clustering but their method was focused on modeling the underlying propensity parameters that lead to binary events. This is different from our approach of adapting synthetic control's weighting philosophy to hazard trajectories under a panel framework with low-rank structure.

\section{Causal Framework: Panel Survival Analysis}\label{sec:method}

In this section, we develop a causal framework for estimating counterfactual survival outcomes using panel data, with a motivating application introduced in Section \ref{sec:introduction}, to evaluate the impact of novel cancer therapies on patient survival. %

\subsection{Setup}\label{ssec:setup}

We begin by describing the causal survival analysis problem in its most fundamental form.  
Our TCL study is naturally panel-like: multiple individuals are observed across countries and lines of therapy. Assume there are $N$ units (e.g. countries) and $P$ periods (e.g. $p \in \{0,1\}$ first-line vs. second-line therapy). %

\paragraph{Data Generating Process and DAG.}

To formalize the data generating process, we describe it in a directed acyclic graph (DAG) shown in Figure~\ref{fig:simple_dag}(a). %
For each unit $n\in [N]$ and period $p \in \{0, 1\}$, treatment assignment $D_{p, n} \in \{0, 1 \} $ acts on the event time $\tau_{p, n} \in \mathbb{R}^+$ (through a hazard process defined below). Unit- and period-level covariates $X_{p, n}$ influence both treatment and outcome. However, we do not observe $\tau_{p, n}$ directly, as it can be masked by the censoring time $C_{p, n} \in \mathbb{R}^+$ if $C_{p, n} < \tau_{p, n}$; we observe time $T_{p, n} = \min\{\tau_{p, n}, C_{p, n}\}$ and the event indicator 
$\Delta_{p, n} = \mathbb{I}\{\tau_{p, n} \leq C_{p, n}\}$. %

\begin{figure}[htbp]
\centering

\begin{subfigure}[t]{0.35\textwidth}
\centering
\resizebox{\linewidth}{!}{%
\begin{tikzpicture}[
  node distance=1cm,
  every node/.style={circle, draw=none, minimum size=0.5cm, align=center},
  solidarrow/.style={->, thick, >=Triangle},
  dashedarrow/.style={->, thick, dashed, >=Triangle}
]
\node[fill=red!40] (X) {$X_{p, n}$};
\node[below=of X, fill=blue!30] (D) {$D_{p, n}$};
\node[right=of D, fill=gray!20] (tau) {$\tau_{p, n}$};
\node[below=of tau, fill=gray!20] (C) {$C_{p, n}$};
\node[right=of tau, fill=green!20] (T) {$T_{p, n}$};

\draw[solidarrow] (D) -- (tau);
\draw[solidarrow] (D) -- (C);
\draw[solidarrow] (tau) -- (T);
\draw[solidarrow] (C) -- (T);

\draw[dashedarrow] (X) -- (D);
\draw[dashedarrow] (X) -- (tau);
\draw[dashedarrow] (X) -- (C);
\end{tikzpicture}%
}
\subcaption{Fully observed confounders}\label{fig:dag-A}
\end{subfigure}%
\hfill
\begin{subfigure}[t]{0.5\textwidth}
\centering
\resizebox{\linewidth}{!}{%
\begin{tikzpicture}[
  node distance=10mm and 12mm,
  every node/.style={circle, draw=none, minimum size=7mm, inner sep=1pt, align=center},
  obs/.style   ={fill=red!40},
  latent/.style={fill=pink!40},
  tauc/.style  ={fill=gray!20},
  treat/.style ={fill=blue!30},
  outc/.style  ={fill=green!20},
  dashedarrow/.style={->, thick, dashed, >=Triangle},
  solidarrow/.style ={->, thick, >=Triangle}
]

\node[tauc]  (Tau) {$\tau_{n,p}$};
\node[treat, left=of Tau]  (D)   {$D_{n,p}$};

\node[tauc,  below=of Tau](C)   {$C_{n,p}$};
\node[outc,  right=of Tau](T)   {$T_{n,p}$};

\node[latent, above=of D] (Un)  {$V_n$};   %
\node[latent, above=of T] (Up)  {$U_p$};   %

\begin{scope}[]
  \coordinate (ctr) at (Tau);
  \draw[line width=0.9pt, color=black!70, rotate around={ 45:(ctr)}]
        (ctr) ellipse [x radius=34mm, y radius=22mm];
  \draw[line width=0.9pt, color=black!70, rotate around={-40:(ctr)}]
        (ctr) ellipse [x radius=34mm, y radius=21mm];
\end{scope}

\node[font=\small, anchor=west,  inner sep=0pt] at ([xshift=5mm]Up.east)  {period $p$};
\node[font=\small, anchor=east,  inner sep=0pt] at ([xshift=-5mm]Un.west) {unit $n$};

\draw[solidarrow] (D)   -- (Tau);
\draw[solidarrow] (D)   -- (C);
\draw[solidarrow] (Tau) -- (T);
\draw[solidarrow] (C)   -- (T);

\draw[dashedarrow] (Un) -- (D);
\draw[dashedarrow] (Un) -- (Tau);
\draw[dashedarrow] (Un) -- (C);

\draw[dashedarrow] (Up) -- (D);
\draw[dashedarrow] (Up) -- (Tau);
\draw[dashedarrow] (Up) -- (C);

\end{tikzpicture}%
}
\subcaption{Unobserved confounder with factor structure}\label{fig:dag-B}
\end{subfigure}

\caption{DAGs of panel data generating processes. Dashed arrows indicate potential confounding paths. Gray and light pink nodes are latent at the time of analysis. $D_{p, n}$: treatment assignment, $\tau_{p, n}$: event time, $C_{p, n}$: censoring time, $T_{p, n}$: observed time. Figure (a) shows the most fundamental case with fully observed covariates $X_{p, n}$; (b) shows the setup where $X_{p, n}$ are unobserved and consist of latent unit factor $V_n$ and latent period factor $U_p$. The unit $n$ ellipse contains all unit-related information while the period $p$ ellipse contains all period-related information.}
\label{fig:simple_dag}
\end{figure}
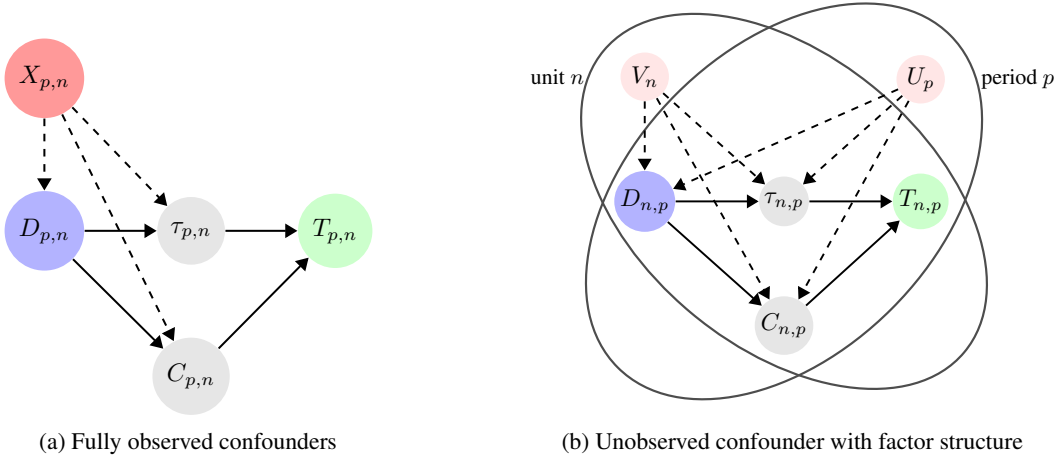

\paragraph{Potential Outcomes.} 

Following the potential outcomes framework of \citet{nayman1923} and \citet{rubin1974}%
, we define the \emph{potential hazard function} 
\[
h^{(d)}_{p, n}(t) := \lim_{\Delta t \to 0}
\frac{\mathbb{P}(t \leq \tau^{(d)}_{p, n} < t+\Delta t \mid \tau^{(d)}_{p, n} \geq t)}{\Delta t}, t \in [0, \widetilde{\tau}]
\] which represents instantaneous risk of event occurrence under treatment $d$ at time $t$ for unit $n$ in period $p$ since time-to-event tracking, as opposed to the typical calendar time. Because late-time survival data are sparse and unstable, we restrict attention to a finite evaluation horizon $\widetilde{\tau} > 0$. 

As such, $h^{(d)}_{p, n}(t)$ induces the potential survival function $S^{(d)}_{p, n}(t) = \exp(-\int_0^t h^{(d)}_{p, n}(s) ds)$ and density $f^{(d)}_{p, n}(t) = h^{(d)}_{p, n}(t) S^{(d)}_{p, n}(t)$ of the potential event time $\tau^{(d)}_{p, n}$. 

In our TCL example, $d=0$ denotes conventional chemotherapy (CC), while $d=1$ corresponds to novel single-agent (SA) therapy. Each patient can only receive one type of therapy in a treatment line and thus, we can never observe all possible potential outcomes – this is the fundamental challenge of causal inference.

\subsection{Causal Estimand}\label{sec:estimand}
Given the model encoded in Figure \ref{fig:dag-A}, we define our causal estimand formally. For a specific unit $n$, our target is $\theta_n(\cdot) = h^{(0)}_{1, n}(\cdot)$, which captures the potential hazard trajectory in the post-period $p=1$ under control ($d = 0$). %
    Returning to our motivating example, the estimand corresponds to the potential hazard trajectory in the \textit{second-line} therapy under the control treatment CC.

For identification and estimation, we have the following assumptions. 

\begin{assumption}[SUTVA]\label{ass:sutva} Fix a period $p\in \{0, 1\}$. 

\noindent(a) No interference.
For every unit $n \in [N]$, treatment level $d\in \{0, 1\}$, and any two assignment profiles $d_{-n},d'_{-n}\in\mathcal{D}^{[N]\setminus\{n\}}$ for all
units other than $n$,
\[
\big (\tau_{p, n}^{(d)} ,  C_{p, n}^{(d)}\big)
\ \text{has the same distribution under}\ (d_{p, n}=d, d_{p, -n})
\ \text{and under}\ (d_{p, n}=d, d'_{-n}).
\]
Equivalently, for each $n$ and $d$, $(\tau_{p, n}^{(d)},C_{p, n}^{(d)})$ depends only on $d_{p, n}$ and not on $d_{p, -n}$.

\smallskip
\noindent(b) Consistency.
If unit $n$ receives treatment $D_{p, n}=d$, then the event time and censoring time coincide with the corresponding potential quantities, i.e.,
\[
\tau_{p, n}  =  \tau_{p, n}^{(d)}
\quad\text{and}\quad
C_{p, n}  =  C_{p, n}^{(d)},
\]
\end{assumption}

 Assumption \ref{ass:sutva} ensures that factual hazards match counterfactual potential hazards under the assigned treatment, and there are no cross-individual spillovers. %

  \begin{assumption}[Non-informative Censoring]\label{ass:censoring}
$\tau_{p, n} \perp C_{p, n}\ \mid (D_{p, n}, X_{p, n})$.
\end{assumption}

Assumption \ref{ass:censoring} allows us to treat censored subjects as if they were removed \textit{at random} from the risk set, so the remaining risk set stays representative \footnote{Note that we assume this mainly for readability and clarity. A natural extension of our current framework using techniques like inverse probability censoring weighting (IPCW)\citep{ipcw} should work if this assumption is relaxed.}.

\subsection{Case 1: No Unobserved Confounding}

In the setting when we have a rich collection of covariates, a natural modeling assumption is that all pertinent covariates (and hence confounders) are observed.  
We first discuss how observed data can be utilized for estimating potential outcomes of interest in a simple single-period setting. This allows us to better understand the data generating process and separate the fundamental difficulties of causal survival analysis from those specific to the panel structure.

\subsubsection{Identification}

For identification, we also need the following positivity assumption.

\begin{assumption}[Positivity]\label{ass:positivity}
    For each period-unit $(p,n)$, any covariate configuration that occurs with positive probability must admit a positive probability of receiving the control treatment $d=0$. Formally, $\Pr(D_{p,n}=0 | X_{p,n}=x)>0 \text{ for all } x \in \mathrm{support}(X_{p,n})$.
\end{assumption}

We now present the identification result.

\begin{proposition}[Identification Under No Unobserved Confounding]\label{prop:identification}
Suppose Assumption \ref{ass:sutva}-\ref{ass:positivity} and no unobserved confounding, $\{\tau^{(0)}_{p, n}, \tau^{(1)}_{p, n}\} \perp D_{p, n} \mid X_{p, n}$. 
Then for some fixed time horizon $\widetilde{\tau}>0$, the marginal potential hazard function is identified as
\[
h^{(0)}_{p, n}(t)  =  
\frac{\mathbb{E}_{X}\!\big[h_{p, n}(t\mid D_{p, n}=0, X_{p, n}) S_{p, n}(t\mid D_{p, n}=0, X_{p, n})\big]}
     {\mathbb{E}_{X}\!\big[S_{p, n}(t\mid D_{p, n}=0, X_{p, n})\big]},  t\in[0,\widetilde{\tau}].
\]
\end{proposition}

See Proof in Appendix \ref{appendix:prop}.

\subsubsection{Estimation} \label{sec:ob_estimation}
\paragraph{Data.} In period \(p\in\{0,1\}\), we observe survival data for each unit (country in our example) \(n\in[N]\) consisting of \(K_{p, n}\) independent observations (patients in our example). For observation \(i=1,\dots,K_{p, n}\), they contribute the assigned treatment $D_{p,n,i}$, features $X_{p, n, i}$ and censored time‐to‐event observation \((T_{p,n,i}, \Delta_{p,n,i})\). %

\paragraph{Time-to-event Modeling.} Based on Proposition \ref{prop:identification}, we boil down the problem of estimating the potential hazard function to estimating the observational conditional hazard $h_{p, n}(t\mid D_{p, n}=0, X_{p, n})$ and survival functions $S_{p, n}(t\mid D_{p, n}=0,X_{p, n})$ . This can be done using standard survival models with certain modeling structures. For example, the Cox model encodes the additional structure assumption that covariates contribute to the hazard function multiplicatively, i.e., 
\begin{equation}\label{eq:cox}
h^{(d)}_{p, n}(t\mid X_{p, n})=\widetilde h(t) \exp (\beta^\top X_{p, n}+\gamma d),
\end{equation}
where $\widetilde h(t)$ is a baseline hazard.

In fact, Equation~\eqref{eq:cox} implies that the log-hazard decomposes additively into a baseline term $\alpha(t)=\log \widetilde h(t)$, covariate contribution $\beta^\top X$, and treatment effect $\gamma d$. More generally, this suggests representing the transformed hazard surface by an additive (or low-rank) decomposition: $
    \alpha(t) + \beta(X) + \gamma(d).
$
Such a representation naturally generalizes to higher-dimensional arrays of transformed hazards indexed by units, periods, treatments, and time, motivating the low-rank structure formulation developed in Section \ref{sec:unobserved} later.

\paragraph{Handling Censoring.}
Note that Assumptions~\ref{ass:censoring} and \ref{ass:positivity} ensure that $h_{p, n}(t\mid D_{p, n}=d,X_{p, n})$ and $S_{p, n}(t\mid D_{p, n}=d,X_{p, n})$ are estimable from right-censored data. For example, if the event-time model has parameter $\eta$, a probability density function $f(\cdot)$ and a survival function $S(\cdot)$, and the censoring model has parameter $\psi$, a probability density function $\widetilde f(\cdot)$ and a survival function $\widetilde S(\cdot)$, then under Assumption \ref{ass:censoring}, we have: for unit $n$, period $p$ and $dt \to 0_+$, 
\begin{align*}
    &\Pr(t \leq T_{p, n} < t+ dt,\ \Delta_{p, n}=1 \mid D_{p, n}, X_{p, n})\\
&\qquad = f_{p, n}(t \mid D_{p, n}, X_{p, n};\eta)  \widetilde S_{p, n}(t \mid D_{p, n}, X_{p, n};\psi)  dt\\
&\Pr(t \leq T_{p, n} < t+ dt,\ \Delta_{p, n}=0 \mid D_{p, n}, X_{p, n})\\
&\qquad = \widetilde{f}_{p, n}(t \mid D_{p, n}, X_{p, n};\psi)   S_{p, n}(t \mid D_{p, n}, X_{p, n};\eta)  dt.
\end{align*}
Hence, likelihood-based survival methods like the Cox model can continue working with the observed-data's %
log-likelihood 
\begin{align*}
  \ell(\eta, \psi)&= \sum_{n,p, i}\Big[
\Delta_{p, n, i}\big\{\log f_{p, n}(T_{p, n, i} \mid D_{p, n, i},X_{p, n, i};\eta)+\log \widetilde S_{p, n}(T_{p, n, i} \mid D_{p, n, i}, X_{p, n, i};\psi)\big\}\\
&+ (1-\Delta_{p, n, i})\big\{\log \widetilde{f}_{p, n}(T_{p, n, i} \mid D_{p, n, i},X_{p, n, i};\psi)+\log  S_{p, n, i}(T_{p, n, i} \mid D_{p, n, i},X_{p, n, i};\eta)\big\}
\Big]\\
&\propto \sum_{n, p, i}\Big[
\Delta_{p, n, i}\log f_{p, n}(T_{p, n, i} \mid D_{p, n, i},X_{p, n, i};\eta)\\
&\qquad\qquad+(1-\Delta_{p, n, i})\log S_{p, n}(T_{p, n, i} \mid D_{p, n, i},X_{p, n, i};\eta)
\Big].
\end{align*}
 Note that the $\psi$-terms and $\eta$-terms nicely decompose. As such, maximizing over $\eta$-terms only is equivalent to maximizing $\ell(\eta, \psi)$ with observed data.

\paragraph{Estimation Procedure.}
We now formally state the estimation procedure for $\theta_n(t) = h^{(0)}_{1, n}(t)$ for all $t$ in some fixed time horizon $[0, \widetilde{\tau}]$ for a specific unit $n$. We introduce the usage of the Cox proportional hazards model and present its consistency result in Section \ref{sec:ob_consistency}. Note that we are utilizing known results in literature for Cox proportional hazards model. In its place,
other methods from literature can be utilized in its place and prior results from literature may provide similar guarantees.

(1) Step 1: Fit a Cox proportional hazards model
\[
h_{1, n}(t \mid D_{1, n},X_{1, n})
=
\widetilde h(t)\exp(\beta^\top X_{1, n} + \gamma D_{1, n}),
\]
obtaining $(\widehat\beta, \widehat {\widetilde \Lambda}(t) := \int_{0}^t \widehat{\widetilde h}(s)ds)$ via the Breslow estimator \citep{breslow1975}.

(2) Step 2: Predict potential conditional hazard and survival for each individual by setting the treatment to be the control.
Given $n$ and indices of observed data $i = 1,\dots, K_{1, n}$, regardless of their observed treatment,
\begin{align}
\widehat h_{1, n, i}^{(0)}(t)
&=\widehat{\widetilde{h}}(t)  \exp(\widehat\beta^{\top} X_{1, n, i}) \text{,
where $\widehat {\widetilde{h}}(t)$ is a derivative of $\widehat {\widetilde \Lambda}_{0,d}(t)$,}
\\
    \widehat S_{1, n, i}^{(0)}(t)
&=
\exp\!\left(
    -\widehat {\widetilde \Lambda}(t) \cdot \exp(\widehat\beta^\top X_{1, n, i})
\right).
\end{align}

(3) Step 3: Plug-in.
\[
\widehat h^{(0)}_{1, n}(t)
=
\frac{
\sum_{i}
    \widehat h_{1, n, i}^{(0)}(t) \widehat S_{1, n, i}^{(0)}(t)
}{
\sum_{i}
    \widehat S_{1, n, i}^{(0)}(t)
}.
\]

\subsubsection{Consistency Result} \label{sec:ob_consistency}
Our estimator for the marginal potential hazard $h_{1,n}^{(0)}(t)$ builds on
the Cox proportional hazards model introduced in Section~\ref{sec:ob_estimation}.
Under standard regularity assumptions and independent censoring,
the Cox partial-likelihood estimator $\widehat{\beta}$ is asymptotically
consistent and asymptotically normal, and the Breslow estimator of the
baseline cumulative hazard $\widehat{\widetilde \Lambda}(t)$ converges uniformly to
$\widetilde \Lambda(t)$ on compact time intervals
\citep{cox1972regression, Tsiatis1981}.
Consequently,
\[
\widehat{S}_{1, n}(t \mid D=0, X_{1,n,i})
=
\exp\!\left(
 -\widehat{\widetilde \Lambda}(t)\exp(\widehat{\beta}^\top X_{1,n,i})
\right)
\xrightarrow{p}
S_{1, n}(t \mid D=0,X_{1,n,i})
\] uniformly in $t \in [0, \widetilde{\tau}]$.
Thus, our estimator of the marginal
counterfactual hazard,
\[
\widehat{h}^{(0)}_{1,n}(t)
=
\frac{
  \sum_i \widehat{h}_{1, n}(t\mid D=0,X_{1,n,i}) \widehat{S}_{1, n}(t\mid D=0,X_{1,n,i})
}{
  \sum_i \widehat{S}_{1, n}(t\mid D=0,X_{1, n,i})
},
\]
inherits consistency by the continuous mapping theorem.

\subsection{Case 2: Unobserved Confounding}\label{sec:unobserved}

In practice, only a subset of the covariates may be observed. The unobserved components confound both treatment assignment and outcomes, making direct identification of causal survival functions challenging, if not infeasible. Moreover, Assumption \ref{ass:positivity} may be violated due to limited availability of certain treatment. These motivate need for additional structure for our setting with unobserved confounding. 

We notice that persistent latent characteristics of units (e.g., country-level differences in healthcare infrastructure, patient mix, or access to medical technology) influence both treatment assignment and outcomes. Similarly, there are persistent latent characteristics of period (e.g., second-line patients have specific biological, disease-specific conditions that are persistent across countries). These characteristics or factors can be considered to be stable over the duration of the study. This suggests {\em factorization} of characteristics as $X_{p, n} = (V_n, U_p)$ where $V_n$ 
represents unit characteristics  or factors and $U_p$ represent period characteristics or factors. We shall consider $V_n$ and $U_p$  to be latent or unobserved. As we discussed near the end of the manuscript, this framework has potential to extend to additional observed covariates and is left for future work.

Because each unit can only adopt one treatment per period, at most one of the potential hazard functions is observed. Estimating counterfactual survival trajectories is therefore equivalent to imputing missing hazards across the unit-period-treatment tuples, which is illustrated in Figure~\ref{fig:observed_tensor}.

\begin{figure}[htbp]
    \centering
    \includegraphics[width=0.65\linewidth]{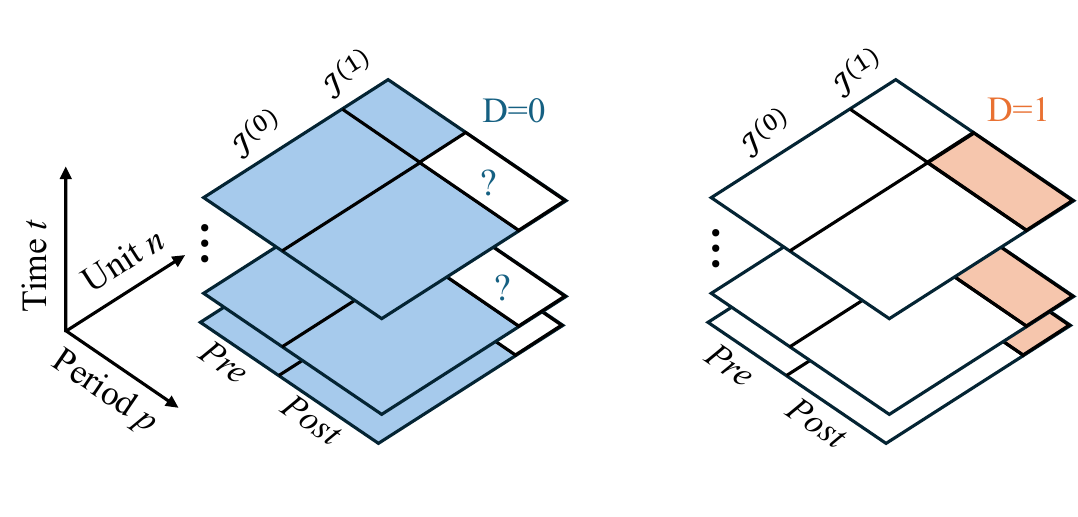}
    \caption{Visualization of the setup. $\mathcal{I}^{(0)}$and $\mathcal{I}^{(1)}$ denotes the control group and treatment group respectively. The shaded area represents the observed data while the question mark region denotes our target.} 
    \label{fig:observed_tensor}
\end{figure}

\paragraph{Factorization of Survival Function.} The distribution of time-to-event $\tau_{n,p}$, i.e. the survival function, 
depends on $X_{p,n}$. We imposed restriction that $X_{n, p}$ separates or factorizes into unit specific component $V_n$ and period
specific component $U_p$. More precisely, we impose the following {\em low-rank} factorization structure on survival function that
factorizes into components that unit specific and period specific. We state the formal assumption first and then provide its 
justification. 
\begin{assumption}[Survival Factor Model]\label{assump:factor_model}
For all $p\in[P]$, $n\in[N]$
and $t\in[0,\widetilde{\tau}]$, the potential survival function under control
$
S^{(0)}_{p,n}(t) = \langle u_p(t),  v_n\rangle,
$
where $u_p: [0,\widetilde{\tau}] \to \mathbb{R}^r$ are (latent) period factors and $v_n\in\mathbb{R}^r$ are (latent) unit factors
for some $r \geq 1$.
\end{assumption}
To understand why Assumption \ref{assump:factor_model} follows naturally once we impose structural separation of $X_{n, p}$ into 
unit specific component $V_n$ and period specific component $U_p$, let us recall the Cox model. In such a model, we have that
logarithm of hazard function having form
\begin{align}
\log h_{p, n}(t\mid D=d, X_{p, n} =(V_n, U_p)) & = \log \widetilde h(t) + \beta^\top X_{p,n} + \gamma d, 
\end{align}
where $\widetilde h(\cdot)$ is a baseline hazard, $\beta$ being model parameter, $\gamma$ capturing treatment effect. 
Writing $X_{p,n} = [V_n   U_p]$ as a vector with corresponding element-wise split of $\beta = [\beta_1   \beta_2]$, we 
have 
\begin{align}
\log h^{(d)}_{p, n}(t) & = \log h_{p, n}(t\mid D=d, X_{p, n} =(V_n, U_p)) \nonumber \\
& = \log \widetilde h(t) + \beta_1^\top V_n + \beta_2^\top U_p + \gamma d. \text{\footnotemark}
\end{align}
\footnotetext{This is similar to two-way fixed effect in econometrics literature.}
Denote $a_n = \exp(\beta_1^\top V_n)$, $b_p = \exp(\beta_2^\top V_p)$, and using the fact that survival function is exponential of negative
integral of hazard function, we obtain (for $d = 0$)
\begin{align}
S^{(0)}_{p, n}(t) & = \exp\Big(-\int_{0}^t h^{(0)}_{p, n}(s) ds\Big) \nonumber \\
& = \exp\Big(-\int_{0}^t \widetilde h(s) \exp\big( \beta_1^\top V_n + \beta_2^\top U_p \big)ds\Big) \nonumber \\
& = \exp\Big(a_n \big(-\int_{0}^t b_p \widetilde h(s)ds\big)\Big) \nonumber \\
& = \exp\Big(a_n \tilde b_p(t) \Big), \label{eq:nonlin.bilin}
\end{align}
where $\tilde b_p(t) = - \int_{0}^t b_p \widetilde h(s)ds$. 

We consider the following elementary fact (see Appendix \ref{appendix:proof_approx} for a proof):
\begin{lemma}\label{lem:approx}
Given $\forall x, |x| \leq B$ and $\varepsilon \in (0,1)$,
\begin{align}
    \Big| e^x - \sum_{k=0}^{r-1} \frac{x^k}{k!}\Big| & \leq \varepsilon, 
\end{align}
as long as $r \geq 5 \max(B, \ln \frac{1}{\varepsilon})$. 
\end{lemma}
This suggests that
\begin{align}
S^{(0)}_{p, n}(t) & \approx \sum_{k=0}^{r-1} \frac{1}{k!} a_n^k \tilde b_p^k(t).
\end{align}
By writing $v_n = [a_n^k: 0\leq k\leq r-1] \in \mathbb{R}^r$ and $u_p(t) = [\frac{1}{k!}\tilde b_p^k(t): 0\leq k \leq r-1]$
with $u_p: \mathbb{R}_{\geq 0} \to \mathbb{R}^r$, we obtain
\begin{align}\label{eq:approx.lr}
S^{(0)}_{p, n}(t) & \approx \langle u_p(t), v_n\rangle.
\end{align}
That is, Assumption \ref{assump:factor_model} is (approximately) satisfied. More generally, for a number of parametric models, it can be checked
that $S^{(0)}_{p,n}(t)$ is a smooth nonlinear transformation of a bilinear form similar to \eqref{eq:nonlin.bilin}. For such settings,
like Lemma \ref{lem:approx}, this would result into approximate low-rank factorization like \eqref{eq:approx.lr} (also see Proposition~1,  \citet{abadie2025causalinferenceframeworkdata}). This provides natural justification for Assumption \ref{assump:factor_model}.

\paragraph{Selection on Latent Factors.}
Analogous to the unconfoundedness assumption in classical causal inference, we require that once latent factors are conditioned upon, treatment assignment is independent of potential hazards:

\begin{assumption}[Selection on Latent Factors]\label{assump:selection_latent}
\[
\mathcal{S} \perp \mathcal{D} \ \mid \ \mathcal{LF},
\]
where $\mathcal{S} = \{S_{p, n}^{(d)}(t)\}_{p, n, d, t}$, realized treatment assignments $\mathcal{D} = \{(p,n): D_{p,n} \text{ observed}\}$, $\mathcal{LF} = \{u_p(\cdot), v_n\}_{p,n}.$
\end{assumption}

Together, Assumptions \ref{assump:factor_model} and \ref{assump:selection_latent} provide the structural basis for imputing missing entries of $\mathcal{S}$ and thereby estimating counterfactual hazard trajectories.

\paragraph{Causal Estimand.}
Building upon the definition in Section \ref{sec:estimand}, our target causal estimand is $
    \theta_{n}(\cdot) = S^{(0)}_{1, n}(\cdot),%
    $ for a given treated unit $n \in \mathcal{I}^{(1)}$. Thus 
it captures the potential transformed hazard trajectory for an treated unit $n \in \mathcal{I}^{(1)}$ under control during the post-treatment period, conditional on latent factors. %
    Returning to our motivating example, the estimand corresponds to the true transformed hazard trajectory from second-line treatment onward had the US continued using CC instead of adopting novel SA.

\subsubsection{Identification}
We now turn to the identification question: under what assumptions can we recover the unobserved counterfactual trajectories $\theta_{n}(\cdot)$ for treated units $n\in \mathcal{I}^{(1)}$?

\begin{assumption}[Linear Span Inclusion]
For each treated unit $n \in \mathcal{I}^{(1)}$, there exists weights $w_{n, m} \in \mathbb{R}^{|\mathcal{I}^{(0)}|}$ such that $v_n = \sum_{m \in \mathcal{I}^{(0)}} w_{n,m} \cdot v_m$.
\label{assump:linearity_span}
\end{assumption}

Assumption \ref{assump:linearity_span} states that within the space of latent unit factors, the treated unit $n$ can be represented as a weighted average of control during the post-intervention period and such weights are period- and time-invariant. 

It is a direct relaxation of the ``convex span'' condition in synthetic control and utilized in prior works such as synthetic interventions. It ensures that counterfactual outcomes for experimental units can be approximated as weighted averages of control group outcomes. In the context of our motivating example, Assumption \ref{assump:linearity_span} requires USA's latent factors to be expressed as a weighted sum of the remaining countries' latent factors in the post-intervention period.

It is worth noting that Assumption~\ref{assump:factor_model} implies that all unit-specific latent factors 
\(\{v_n\}_{n=1}^N \subset \mathbb{R}^r\) 
lie in an \(r\)-dimensional subspace for some fixed \(r\). 
Let \(\mathcal{V}_i = \mathrm{span}\{v_1, \ldots, v_i\}\) denote the subspace generated by the first \(i\) vectors. 
Then the sequence \((\dim \mathcal{V}_i)_{i=1}^N\) is nondecreasing and bounded above by \(r\). 
Hence, there exist at most \(r\) indices \(i_1, \ldots, i_r \in [N]\) such that $
\dim(\mathcal{V}_{i_j}) = \dim(\mathcal{V}_{i_j-1}) + 1,$
i.e., vectors \(v_{i_j}\) that introduce new linearly independent directions. 
Equivalently, all but at most \(r\) of the vectors \(\{v_n\}_{n=1}^N\) lie in the span of their predecessors. Thus, with $N$ large enough, Assumption \ref{assump:linearity_span} should be satisfied for almost all (with exception of $r$ units in the worst case). 

We now state the identification result.

\begin{theorem}\label{theorem:theorem1}
    Given the assumptions above, we have 
    \begin{align}
         S^{(0)}_{1, n}(\cdot) =  \sum_{m \in \mathcal{I}^{(0)}}  w_{n, m} S^{0}_{1, m}(\cdot).
    \end{align}
    where $h^{(d)}_{p, m}(\cdot)$ and $h^d_{p, m}(\cdot)$ are the counterfactual and factual hazard function of unit $m$ in period $p$ under treatment $d$.
\end{theorem}
See proof in Appendix \ref{appendix:theorem1}. This means the counterfactual post-treatment transformed hazard for treated unit $n$ under control is identified by a linear combination of post-treatment transformed hazards from control units, using weights determined by pre-treatment similarity in latent space.

\subsubsection{Synthetic Survival Control (SSC) Estimator}\label{ssec:ssc}

\paragraph{Data.}
In period \(p\in\{0,1\}\), we observe survival data for each unit \(n\in[N]\) consisting of \(K_{p, n}\) independent observations. For observation \(i=1,\dots,K_{p, n}\), they contribute the assigned treatment $D_{p,n,i}$ and censored time‐to‐event observation \((T_{p,n,i}, \Delta_{p,n,i})\). %

\paragraph{Estimation Procedure.}
Consider a treated unit $n \in \mathcal{I}^{(1)}$ of interest. Let $\mathcal I^{(0)}$ be $N_0$ control donor units and $\mathcal I^{(1)}$ the treated unit(s).

\medskip
(1) Time-to-event estimation.
For every donor $m\in\mathcal I^{(0)}$ and both periods $p\in\{0,1\}$, and for the target treated unit $n$ in the pre–period $p=0$, fit a Kaplan–Meier estimator to obtain
\[
\widehat{S}^0_{p,l}(t),\quad t\in[0,\widetilde\tau],\qquad
l\in\{n\}\cup\mathcal I^{(0)}.
\]
This inherently handles censoring in patient-level data $\{T_{p,n,i}, \Delta_{p,n,i}\}$ to estimate the underlying hazard functions for each unit-period combination. The accuracy of these estimates depend on the number of observations $K_{p,n}$ within each unit-period. See Section \ref{sec:KM} for more details of time-to-event estimation with Kaplan-Meier empirically.

\noindent (2) Common grid subsampling.
Choose a finite grid $\mathcal T=\{t_1,\ldots,t_{T_0}\}\subset[0,\widetilde\tau]$ and subsample from the estimated survival functions to form vectors for all $l\in\{n\}\cup\mathcal I^{(0)}$, 
\[
\widehat{S}_{p,l} := \big(\widehat{S}_{p,l}(t_1),\ldots,\widehat{S}_{p,l}(t_{T_0})\big)^\top .
\]

\noindent (3) Weights from the pre–period. Stack all control units’ pre–period columns into the $T_0\times N_0$ matrix
$\widehat{S}_{0,\mathcal I^{(0)}}:=[\widehat{S}_{0,m}]_{m\in\mathcal I^{(0)}}$ and the treated unit’s pre vector $\widehat{S}_{0,n}(\cdot)$. 
We wish to utilize the Principal Component Regression (PCR) algorithm (cf. \citet{agarwal2024syntheticinterventions}) to learn weights. 
Specifically, compute the top-$r_0$ singular value decomposition of $\widehat{S}_{0,\mathcal I^{(0)}}$ (with $r_0$ chosen by a gap rule, elbow, or cross-validation as explained in \citet{agarwal2024syntheticinterventions}):
\[
\widehat{S}_{0,\mathcal I^{(0)}} = \sum_{i=1}^{r_0} \widehat{s}_i \widehat{u}_i \widehat{v}_i^\top,
\qquad
\widehat{V}_0 := [\widehat{v}_1,\ldots,\widehat{v}_{r_0}] \in \mathbb{R}^{N_0 \times r_0}.
\]

We estimate weights
$\widehat{w}$ by projecting onto the rank-$r_0$ space and imposing the constraint:
\[
\widehat{w}
 \in 
\arg\min_{w\in\mathbb R^{N_0}}
\big\|
\widehat{S}_{0,n} - \widehat{S}_{0,\mathcal I^{(0)}} w
\big\|_2^2
\quad\text{s.t.}\quad
(I - \widehat V_0\widehat V^{\top}_0)w = 0,
\] which has a closed form solution $\widehat{w}
=
\left(
\sum_{i=1}^{r_0}
1/{\widehat{s}_{i}} 
\widehat{v}_{i} \widehat{u}_{i}^{\top}
\right)
\widehat{S}^0_{0,n}.$

\medskip
\noindent {(4) Post--period extrapolation and inverse transform.}
Form the counterfactual transformed trajectory for the experimental unit $n$ in the
post period:
\[
\widehat{S}^{(0)}_{1,n}(t)
=
\sum_{m\in\mathcal I^{(0)}} \widehat{w}_{n,m} \widehat{S}^0_{1,m}(t),
\qquad t\in[0,\widetilde\tau].
\]

In the context of TCL example, the estimator learns a set of weights that minimizes pre-intervention discrepancy between the US and a weighted combination of countries that maintained CC throughout the study period after pre-processesing $\widehat{S}_{0,\mathcal I^{(0)}}$ by obtaining its rank-$r_0$ approximation. Then this weighted average forms our synthetic USA. Accordingly, the estimator applies the learned weights to the post-intervention outcomes of the control countries to estimate the counterfactual hazards trajectory of the US under continued CC after drug offering shift. 

Note that in this work, we introduce the use of the Kaplan-Meier estimator for the first step time-to-event estimation. The subsequent finite-sample consistency guarantee is also focused on this approach. Other time-to-event estimators can be applied as well and their finite-sample consistency property analyses are left for future works.

\subsubsection{Finite Sample Consistency}
\label{sec:theory}

We analyze the consistency of our SSC estimator. 

First, we state additional assumptions for such analysis. Let \( \mathcal{S}_{0, \mathcal{I}^{(0)}} \in \mathbb{R}^{T_0 \times N_0} \) be the pre-intervention matrix for units remaining under control, where $N_0=|\mathcal{I}^{(0)}|$ is the number of control units and $T_0$ is the number of evaluation timestamps in $ \mathcal{T} \subseteq [0, \widetilde{\tau}]$. We also note that the survival functions $S_{p,n}^{(0)}(\cdot)$ are naturally bounded between 0 and 1.

\begin{assumption}[Well-conditioned Spectrum]
\label{ass:spectrum}
The expectation of the control group pre-treatment matrix satisfies:
$
\text{cond}(\mathcal{S}_{0, \mathcal{I}^{(0)}}) \leq \xi', \text{ and } \| \mathcal{S}_{0, \mathcal{I}^{(0)}} \|_F^2 \geq \xi'' N_0T_0
$ for some constants $\xi', \xi''>0$, where $\|\cdot\|_{F}$ is the Frobenius norm.
\end{assumption}

\begin{assumption}[Latent time-period span]
\label{ass:span}
We assume the post-period specific factors lie in the span of pre-period specific factors, i.e., $\forall t \in  \mathcal{T}, u_{1}(t) \in \text{span}(\left\{ u_{0}(s): s \in  \mathcal{T}\right\})$ conditional on $\{\mathcal{LF}, \mathcal{D}\}$.
\end{assumption}

We now state the main result characterizing the consistency of our estimator.

\begin{theorem}[Consistency of Counterfactual Transformed Hazard Estimation]
\label{thm:consistency}
Fix a treated unit \( n \in \mathcal{I}^{(1)} \).
Let \(  \widehat{\theta}_n(t):= \widehat S^{(0)}_{1, n}(t) \) denote the survival synthetic control estimate of the counterfactual transformed hazard \( \theta_n(t):= S^{(0)}_{1, n}(t)\), $K=\min_{p, n}\{K_{p, n}\}$ be the minimum number of observations per unit and period across all units and periods, $T_0$ be the number of evaluation timestamps in $\mathcal{T} \subseteq [0, \widetilde{\tau}]$, and $r_0$ be defined in the estimtaion procedure. 
Then, if $K> c\max\{T_0, N_0\}(\log N_0 + \log \max\{T_0, N_0\})$ for some constant $c$,
\[
\sup_{t \in \mathcal{T}}\big|\widehat\theta_n(t)-\theta_n(t)\big|
 = 
\mathcal O_p\!\Bigg( 
r_{0}^{3/4} T_0^{-1/4}
 + 
r_{0}^{2} 
\max\!\Big\{N_0^{1/2}T_0^{-3/2}, T_0^{-1/2}, N_0^{-1/2}\Big\}
\Bigg).
\]
\end{theorem}

\noindent See proof in Appendix \ref{sec:proof_thm2}. 

\section{Synthetic Validation}\label{ssec:synthetic}

To begin with, we consider a simulation setting where we generate observations as per the data generating model introduced in this work. As argued, the SSC method should be able to provide a consistent estimator for such a setting, and we can verify it given the prior knowledge of the data generating process. As we do so, we shall aim to validate a few aspects. First, SSC method is able to provide consistent estimation across a number of popular parametric models from survival analysis literature. Since the SSC method is agnostic to the underlying data generating mechanism, this would suggest robustness of SSC across different types of data characteristics. Second, understanding the role of different aspects of the data on method performance. While theoretical results provide performance guarantees, such data-driven analysis helps set guidelines for using SSC in empirical studies.

\paragraph{Simulation Setup.} 

In each simulation, we consider $N=20$ units in the same panel setup introduced in Section~\ref{sec:method}, with one unit designated as the treated unit that switches from the control treatment $D=0$ in the pre-period $p=0$ to the intervention treatment $D=1$ in the post-period $p=1$, and all other units remaining under control ($D=0$) in both periods. For each period--unit pair $(p, n)$ we generate $K_{p,n}\in\{100,300,700\}$ independent samples of right–censored time–to–event data
\[
(T_{p,n,i},\Delta_{p,n,i})
=
\big(\min\{\tau_{p,n,i},C_{p,n,i}\}, \mathbf{1}\{\tau_{p,n,i}\le C_{p,n,i}\}\big),
\]
where censoring times $C_{p,n,i}\sim\mathrm{Exp}(\nu)$ are independent of $\tau_{p,n,i}$ conditional on $\big(D_{p,n},U_p,V_n\big)$, satisfying the independent censoring Assumption~\ref{ass:censoring}. We set the censoring rate to be below $10\%$ by choosing $\nu$ so that the empirical censoring proportion is approximately $10\%$ across designs. Latent factors $(V_n,U_p)$ are drawn i.i.d.\ from $\mathcal N(0,I_r)$ with $r=4$ and then fixed throughout the simulation. 

For the event-time models, we consider two canonical survival specifications.

\emph{1.\ Cox proportional hazards model.}
For each unit $n$ and period $p\in\{0,1\}$, we generate potential event times under control from
\[
h^{(0)}_{p,n}(t)
=
\widetilde h(t)\exp\!\big(\beta_1^\top V_n + \beta_2^\top U_p\big),
\qquad \widetilde h(t)\equiv \lambda > 0,
\]
with $\lambda$ fixed (we take $\lambda = 0.05$ in all experiments) and the coefficient vectors $\beta_1,\beta_2\in\mathbb R^r$ are drawn as $\beta_1,\beta_2 \overset{\text{i.i.d.}}{\sim} \mathcal N(0,I_r)$. The implied survival function under control is thus
\[
S^{(0)}_{p,n}(t)
=
\exp\big[-\lambda t \exp\big(\beta_1^\top V_n + \beta_2^\top U_p\big)\big],
\]
and potential event times $\tau^{(0)}_{p,n}$ are sampled via inverse transform:
\[
\tau^{(0)}_{p,n}
=
\frac{-\log W}{\lambda \exp(\beta_1^\top V_n + \beta_2^\top U_p)},
\qquad W\sim\mathrm{Unif}(0,1).
\]

\emph{2.\ Aalen’s additive hazards model.}
For each unit $n$ and period $p\in\{0,1\}$, we generate potential event times under control from
\[
h^{(0)}_{p,n}(t)
=
\widetilde h'(t)+\beta_1'^\top V_n+\beta_2'^\top U_p,
\qquad \widetilde h'(t)\equiv \lambda_0 > 0.
\]
with the coefficient vectors $\beta_1',\beta_2'\in\mathbb R^r$ are drawn as $\beta_1',\beta_2' \overset{\text{i.i.d.}}{\sim} \mathcal N(0,I_r)$. To guarantee nonnegative hazards, we choose
$
\lambda_0
=
0.05 - \min_{p\in\{0,1\}, n\in[N]}
\big(\beta_1'^\top V_n + \beta_2'^\top U_p\big)
$. Thus, the potential event times $\tau^{(0)}_{p,n}$ are sampled via inverse transform:
\[
\tau^{(0)}_{p,n}
=
\frac{-\log W}{\lambda_0 + \beta_1'^\top V_n + \beta_2'^\top U_p},
\qquad W\sim\mathrm{Unif}(0,1).
\]

\paragraph{Evaluation metrics.}\label{subsec:metrics}
For the treated unit $n$, let 
$\theta_n(t)=S^{(0)}_{1,n}(t)$ denote the true post–period control survival trajectory based on the true data generating process and 
$\widehat{\theta}_n(t)$ its SSC estimate, both evaluated on a common time grid 
$\mathcal T=\{t_1,\ldots,t_{T_0}\}\subset[0,\widetilde{\tau}]$. Specifically, the evaluation horizon is set to be $
\widetilde{\tau}
:=\mathrm{quantile}_{0.90}\!\left(\{T_{p,n,i}\}\right),
$
which trims extreme right tails where empirical support is sparse, stabilizing the downstream survival estimates. And the grid is then chosen as $T_0=100$ equally spaced timestamps
$\mathcal T=\mathrm{linspace}(0,\widetilde{\tau},T_0)$. 
We quantify estimation accuracy using the \emph{sup–norm error}
\[
\|\widehat{\theta}_n-\theta_n\|_{\infty}
:=\max_{t\in\mathcal T}\big|\widehat{\theta}_n(t)-\theta_n(t)\big|,
\]
which directly captures the worst–case deviation between the estimated and true trajectories.

\begin{figure}[t]
    \centering
    \includegraphics[width=\linewidth]{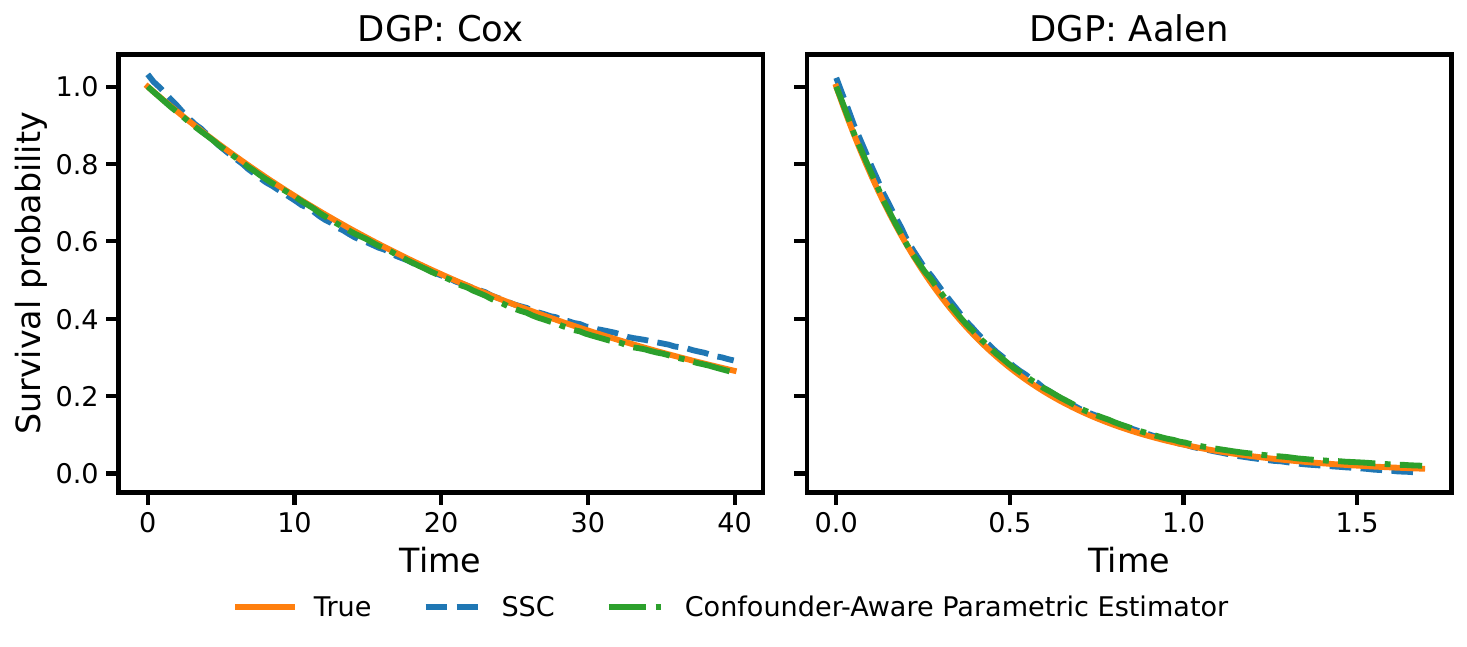}
    \caption{Counterfactual survival estimation under two DGPs.
    Post-period control survival trajectory for the treated unit under:
    (i) the true DGP, 
    (ii) SSC, and 
    (iii) the Confounder-Aware Parametric Estimator which observes the true latent factors and fits the correctly specified parametric survival model.
    SSC closely approximates the true curve in both Cox and Aalen models.}
    \label{fig:ssc_vs_oracles}
\end{figure}

\begin{figure}[t]
    \centering
    \includegraphics[width=\linewidth]{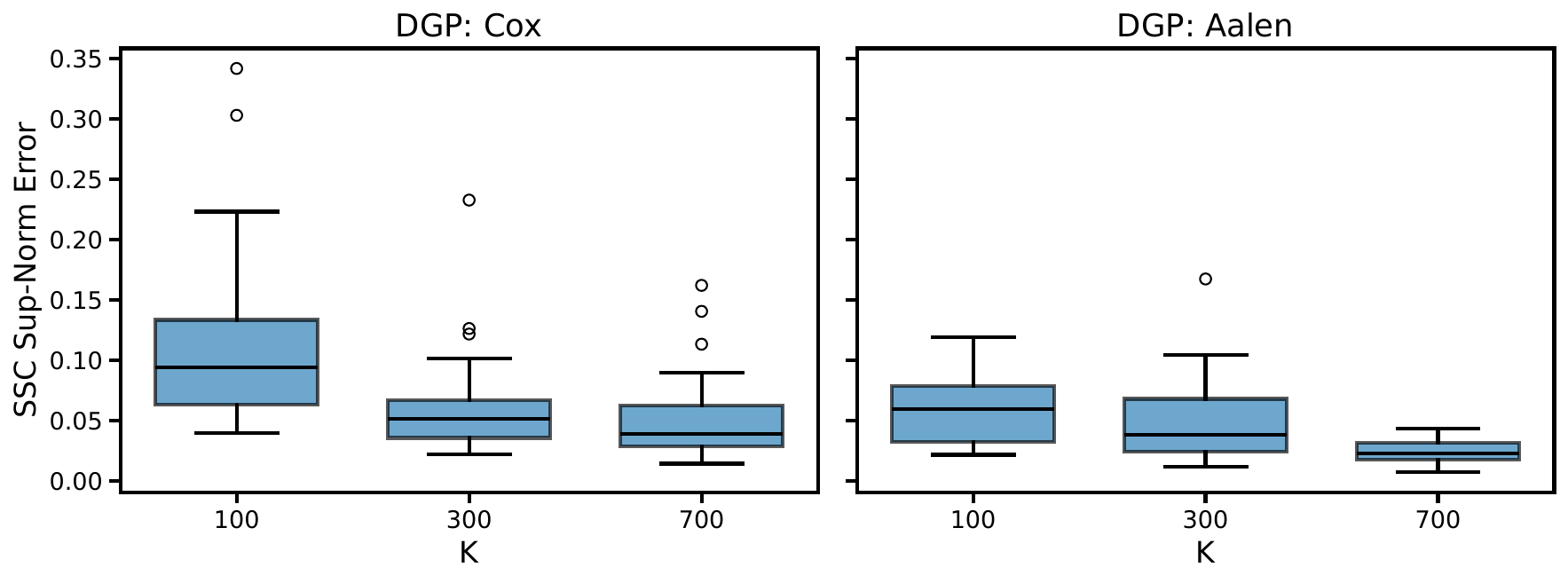}
    \caption{{SSC sup-norm errors across sample sizes.}
    Distribution of sup-norm estimation errors for SSC under Cox and Aalen DGPs over 20 simulations each.
    Errors decrease rapidly with increasing sample size $K$, and variability contracts accordingly.}
    \label{fig:ssc_boxplots}
\end{figure}

\begin{table}[t]
\centering
\caption{Sup-norm error summary for Cox and Aalen DGPs.
Reported are mean $\pm$ standard deviation across simulations.}
\label{tab:synthetic_results}
\begin{tabular}{lccc}
\toprule
DGP & $K$ &
\begin{tabular}{c}SSC Error \\ (mean $\pm$ sd)\end{tabular} &
\begin{tabular}{c}Confounder-Aware Parametric \\ Estimator Error (mean $\pm$ sd)\end{tabular}
\\
\midrule
Cox   & 100 & 0.1177 $\pm$ 0.0835 & 0.0247 $\pm$ 0.0121 \\
Cox   & 300 & 0.0652 $\pm$ 0.0497 & 0.0204 $\pm$ 0.0096 \\
Cox   & 700 & 0.0542 $\pm$ 0.0413 & 0.0194 $\pm$ 0.0104 \\[4pt]

Aalen & 100 & 0.0621 $\pm$ 0.0307 & 0.0185 $\pm$ 0.0052 \\
Aalen & 300 & 0.0507 $\pm$ 0.0388 & 0.0136 $\pm$ 0.0067 \\
Aalen & 700 & 0.0245 $\pm$ 0.0102 & 0.0099 $\pm$ 0.0044 \\
\bottomrule
\end{tabular}
\end{table}

\paragraph{Evaluation results.}
Figure~\ref{fig:ssc_vs_oracles} compares, in one example simulation with $K=700$ observations for each period-unit pair, the post-treatment control survival trajectory of the treated unit from the true data generating proces, SSC estimator and Confounder-Aware Parametric Estimator, which observes the true latent factors and fits the correctly specified parametric survival model. In both the Cox and Aalen settings, the SSC estimator closely matches the ground-truth trajectory. The agreement is particularly notable because SSC does not rely on any knowledge of the underlying survival model or the latent confounders. Its estimates also closely track those of the Confounder-Aware Parametric Estimator. This demonstrates that SSC can recover the correct counterfactual trajectory even when the data are generated from structurally distinct models and fully latent confounding variables that are agnostic to the SSC estimator itself.

We repeat the simulation process for 20 times for $K = \{100, 300, 700\}$ each respectively and Figure~\ref{fig:ssc_boxplots} reports the distribution of sup-norm errors for SSC as the number of observations per unit-period increases. Both DGPs exhibit a clear and rapid reduction in error as $K$ grows: both the median error and variability shrink substantially from $K=100$ to $K=700$. 

Table~\ref{tab:synthetic_results} summarizes the mean and standard deviation of SSC’s sup-norm error across repeated simulations. In both DGPs, SSC achieves diminishing average errors as $K$ increases. The Confounder-Aware Parametric Estimator consistently attains lower error, as expected, since it is given access to the true confounders and the correct parametric form, but the performance gap narrows quickly with increasing $K$. This convergence in accuracy highlights SSC’s ability to fully exploit the latent low-rank structure without requiring model knowledge or observing the confounders themselves.

Overall, the synthetic experiments demonstrate that SSC is accurate, robust across model classes, and exhibits the expected improvement with additional data.

\section{Empirical Application: Evaluating Novel Therapies for T-cell Lymphoma}\label{sec:empirics}

  To complement the synthetic validation, we now demonstrate the use of SSC in our motivating clinical example for evaluating the second-line therapies for relapsed and refractor TCL. This application fits precisely into our SSC framework as all countries prescribe CC in both lines of treatment except the USA, as the treated unit, switches to prescribe the novel SA in second-line. Our multi-country dataset naturally supplies a donor pool of control units so that we can use our SSC estimation procedure to leverage the control units to construct a synthetic control for the USA under CC in the second-line and compare it against the factual USA receiving the novel SA treatment. This section therefore applies SSC to quantify the superiority of treatment efficacy of the novel SA in TCL and to assess estimator fidelity using a held-out USA-CC cohort.

\subsection{Data}
We use a retrospective clinical dataset on 925 patients with relapsed and refractory TCL from 13
institutions across 10 countries. The data includes time-to-event information and  treatment assignments
(SA, as intervention, or CC, as control). Follow-up occurred from
the date of cancer diagnosis until death, loss to follow-up, or end of study.

The analysis defines the first-line treatment period as pre-period ($p=0$) and the second-line
treatment as post-period ($p=1$). The treatment assignment $D_{p,n}$ indicates whether a patient
received SA ($D_{p,n}=1$) or CC ($D_{p,n}=0$) in period $p$. For each patient, the observed time $T_{p,n,i}$ is the minimum of the time-to-event ($\tau_{p,n,i}$) and
time-to-censoring ($C_{p,n,i}$). The censoring indicator is $\Delta_{p,n,i}$ = 1 if $\tau_{p,n,i} \leq C_{p,n,i}$, and 0 otherwise. Since the target outcome is failure of treatment,
in the pre-period, the event is death or start of second-line treatment, and censoring
is loss-to-followup; in the post-period, the event is death or start of third-line treatment,
and censoring is loss-to-followup.

\subsection{Analysis and Results}
Our empirical estimand is the counterfactual second-line survival trajectory for the USA had it continued prescribing CC rather than adopting SA. Applying the SSC procedure, we estimate the USA-SA cohort's counterfactual survival under CC in second-line and report the pre- and post-period survival curves in Figure~\ref{fig:loghazard}. The factual USA trajectory under SA (denoted in orange line) exhibits consistently higher survival than its synthetic control (denoted in blue dash line), suggesting superiority in efficacy of SA relative to CC.

\begin{figure}[htbp]
\centering
\includegraphics[width=\textwidth]{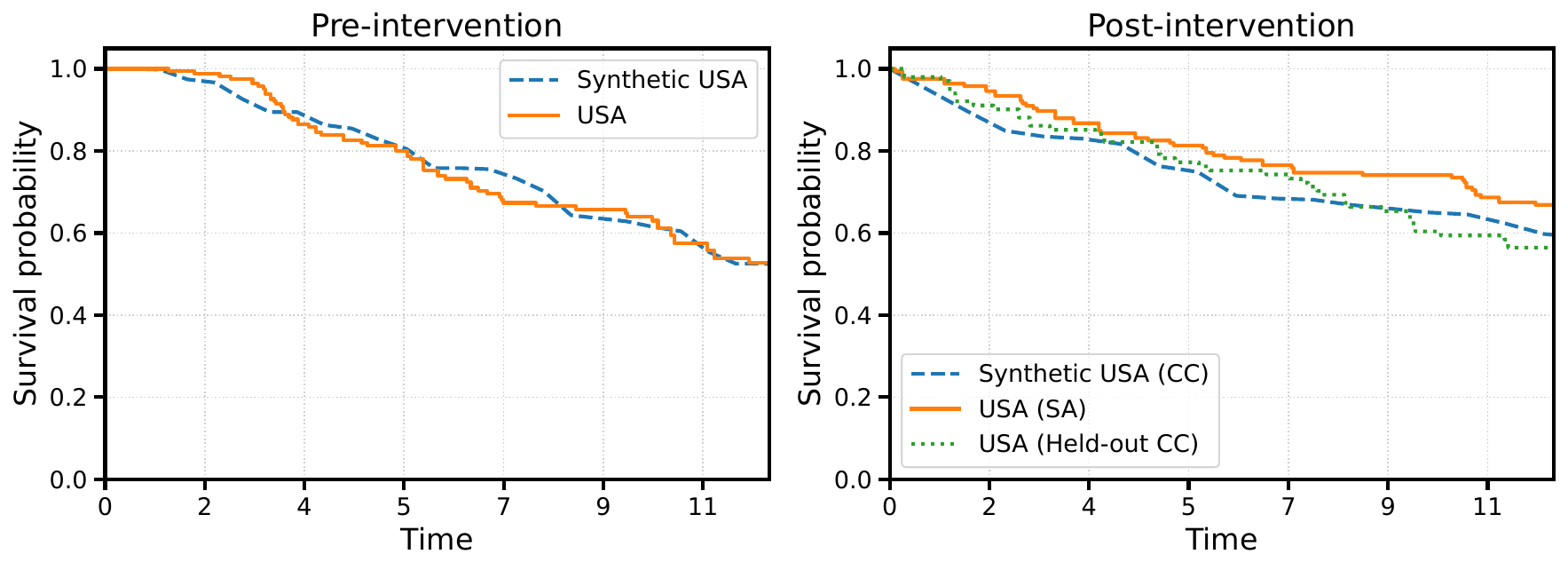}
\caption{Pre-/Post-treatment survival function: USA factual (SA in second-line, orange), synthetic counterfactual under CC (blue), USA held-out (CC in second-line, green).}
\label{fig:loghazard}
\end{figure}

To validate the estimator, we exploit the fact that we observe a cohort of USA patients who received
CC in the second-line treatment. A necessary step is to establish that the held-out USA-CC and USA-SA cohorts do not differ systematically at baseline. We compare the two cohorts along a comprehensive set of prognostic, clinical, demographic, and laboratory variables, including International Prognostic Index (IPI), Prognostic Index for T-cell lymphoma (PIT), disease characteristics, and standard laboratory markers. Continuous covariates are assessed using two-sample $t$-tests and binary covariates using chi-squared tests. In addition, we compute standardized mean differences (SMDs), interpreting SMD $<0.1$ as negligible imbalance and SMD $<0.2$ as acceptable. Across all covariates, hypothesis tests yielded $p$-values $>0.1$, and all SMDs were below $0.2$ (with the majority below $0.1$), indicating that the two cohorts are well balanced at baseline. This comparability enables the held-out USA-CC cohort to serve as a real-world benchmark for assessing estimator fidelity. We also note that the USA-CC cohort is completely hidden from our SSC estimator, thus no information leakage.  
Figure~\ref{fig:loghazard} shows that SSC closely matches the true USA--CC trajectory in the second-line, supporting the credibility of the synthetic control estimation.

To further assess robustness, we perform $500$ bootstrap resamples of the donor pool and recompute the synthetic USA-CC counterfactual for each resample. The resulting 95\% bootstrap confidence interval is shown in Figure~\ref{fig:bootstrap} alongside the ground truth USA-CC trajectory. The synthetic estimates remain stable across resamples and the confidence band closely overlaps the held-out USA-CC curve throughout the post-period. This bootstrap analysis highlights the robustness of SSC to sampling variability and provides additional empirical support for the accuracy of our estimator in a real clinical setting.

\begin{table}[htbp]
\centering
\caption{Comparability of baseline characteristics between USA SA and CC cohorts. Values are reported as mean (standard deviation) unless otherwise noted. There is no statistical evidence of imbalance between the USA CC and SA cohorts across a broad set of clinically relevant baseline characteristics.}
\label{tab:comparability}
\begin{tabular}{lllll}
\toprule
Variable & USA SA mean (std) & USA CC mean (std) & $p$-value & SMD \\
\midrule
\multicolumn{5}{l}{\textbf{Prognostic Score}} \\
IPI Score & 2.52 (0.98) & 2.58 (0.99) & 0.715 & 0.06 \\
PIT Score & 1.65 (1.03) & 1.76 (1.03) & 0.518 & 0.10 \\
\midrule
\multicolumn{5}{l}{\textbf{Clinical}} \\
Histology: AITL (\%) & 38.2 (48.6) & 31.7 (46.5) & 0.337 & 0.14 \\
Mean duration 1st line (yrs) & 0.94 (1.17) & 0.96 (1.09) & 0.906 & 0.01 \\
Refractory (\%) & 41.2 (49.2) & 39.8 (48.9) & 0.923 & 0.03 \\
Extranodal $>$1 (\%) & 24.1 (42.8) & 26.9 (44.4) & 0.707 & 0.06 \\
Stage & 3.29 (0.88) & 3.31 (0.92) & 0.861 & 0.02 \\
ECOG Score & 0.81 (0.74) & 0.89 (0.85) & 0.407 & 0.11 \\
\midrule
\multicolumn{5}{l}{\textbf{Demographics}} \\
Age $>$ 60 (\%) & 54.1 (49.8) & 44.2 (49.7) & 0.143 & 0.20 \\
Race: White (\%) & 0.79 (0.41) & 0.78 (0.42) & 0.765 & 0.04 \\
Sex: Female (\%) & 37.6 (48.5) & 30.8 (46.2) & 0.304 & 0.14 \\
\midrule
\multicolumn{5}{l}{\textbf{Laboratory}} \\
Ki67 $\geq 40$\% & 30.6 (46.1) & 31.7 (46.5) & 0.949 & 0.02 \\
Abnormal ALC & 41.9 (49.3) & 37.8 (48.5) & 0.679 & 0.08 \\
Abnormal Albumin & 38.8 (48.7) & 40.0 (49.0) & 0.973 & 0.02 \\
Abnormal HepB\_cAb & 2.33 (0.51) & 2.31 (0.53) & 0.869 & 0.03 \\
\bottomrule
\end{tabular}
\end{table}

\begin{figure}[!h]
    \centering
    \includegraphics[width=0.6\linewidth]{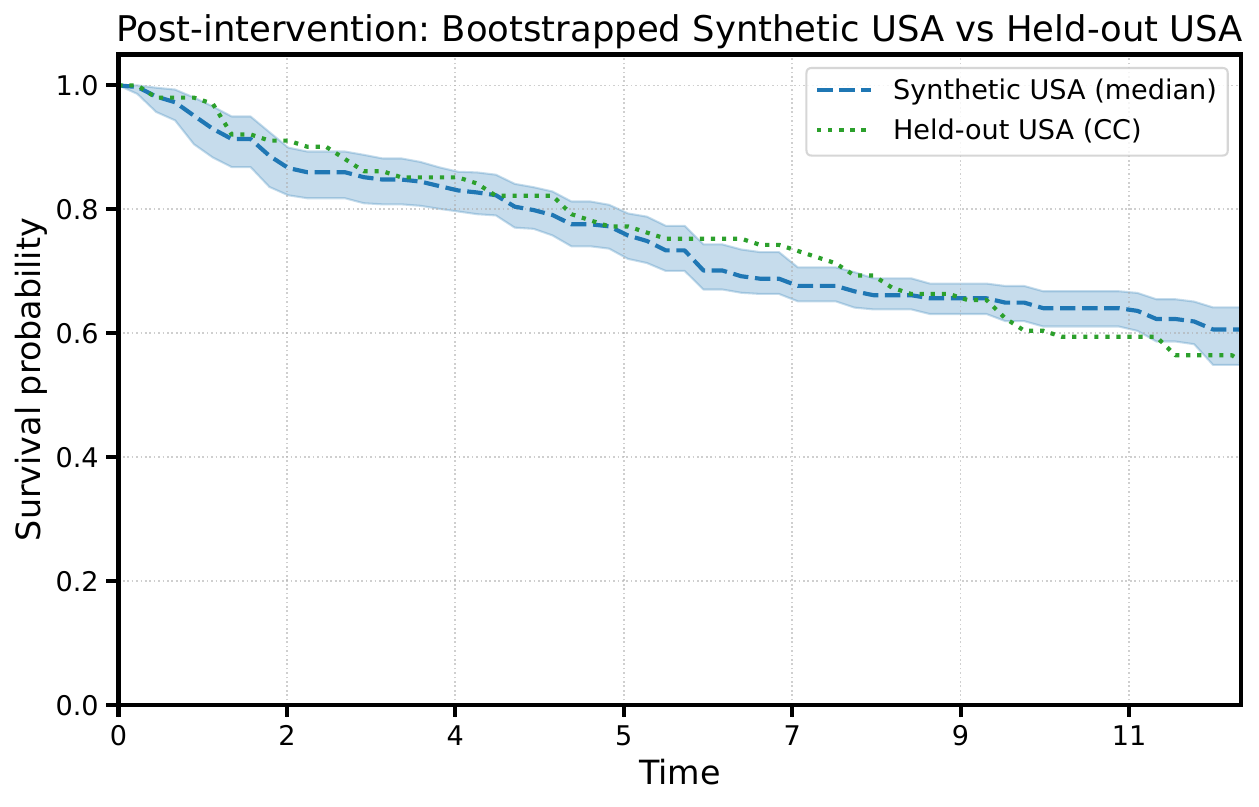}
    \caption{Bootstrapped 95\% CI for synthetic USA SA cohort under CC, alongside the USA CC cohort ``ground truth'', in post-period. The shaded blue region represents the 95\% CI.}
    \label{fig:bootstrap}
\end{figure}

\section{Conclusion}\label{sec:conclusion}
This paper develops a causal framework for survival analysis in panel settings and applies it to evaluate impact of novel TCL therapies. Our method extends synthetic control techniques to the hazard domain, enabling counterfactual estimation with observational panel data.

Empirically, we find that access to novel SA therapies in the US is associated with a significant reduction in post-treatment hazard rates compared to a synthetic control constructed from countries that continue to rely on CC. This suggests that SA may provide a meaningful clinical benefit in prolonging survival among patients with aggressive TCL subtypes. Our framework also enables rigorous validation: by comparing synthetic USA under CC to actual USA-CC patient outcomes, we show that the synthetic control reproduces key survival trends in held-out data. This supports the internal validity of the estimated counterfactual survival.

Several limitations merit discussion. First, we assume non-informative censoring, which may not hold in all clinical settings. Extensions using IPCW \citep{ipcw} or joint modeling should help relax this. Second, we do not directly leverage the observed covariates in the unobserved confounding panel setup. Future work could explore extensions to time-dependent treatment regimes and the incorporation of high-dimensional covariates. We expect our framework to be extended to applications in other domains where timing matters.

\clearpage
\bibliographystyle{plainnat}   %
\bibliography{css}      %

\clearpage
\appendix
\thispagestyle{empty}

\onecolumn
\section{Proof of Proposition \ref{prop:identification}}\label{appendix:prop}
\begin{proof}
All the following proof is for a specific unit $n$ and period $p$ and thus we omit the subscripts for simplicity.

By the consistency and exchangeability assumption,
\begin{align}
S^{(d)}(t) &= \mathbb{E}_X[S^{(d)}(t\mid D=d,X)] = \mathbb{E}_X[S(t\mid D=d,X=x)]  \label{eq:s}\\
f^{(d)}(t) &= \mathbb{E}_X[f^{(d)}(t\mid X=x)] = \mathbb{E}_X[f(t\mid D=d,X=x)] \label{eq:f}.
\end{align}

Importantly, unlike $S^{(d)}(t)$ or 
$f^{(d)}(t)$, as the instantaneous rate of event occurrence given that the event has not happened, $h^{(d)}(t) = \frac{f^{(d)}(t)}{S^{(d)}(t)}$ does not decompose linearly over covariates $X$; that is, the marginal hazard 
$h^{(d)}(t)\neq \mathbb{E}_X[h(t | D=d,X)]$. Instead, it follows from \eqref{eq:s} and \eqref{eq:f} that
\[
h^{(d)}(t) 
= \frac{f^{(d)}(t)}{S^{(d)}(t)}
= \frac{\mathbb{E}_X[f(t\mid D=d,X)]}
       {\mathbb{E}_X[S(t\mid D=d,X)]}=\frac{\mathbb{E}_X[h(t\mid D=d,X) S(t\mid D=d,X)]}
       {\mathbb{E}_X[S(t\mid D=d,X)]}.
\]

\end{proof}

\section{Proof of Lemma \ref{lem:approx}}\label{appendix:proof_approx}
The Lagrange form of the Taylor remainder gives, for some $c$ between $0$ and $x$,
\[
R_r(x)
= e^x - \sum_{k=0}^{r-1} \frac{x^k}{k!}
= \frac{e^{c}}{r!} x^{r}.
\]
Hence, for all $|x|\le B$,
\[
|R_r(x)| \le \frac{e^{|c|}}{r!} |x|^{r} \le \frac{e^{B}}{r!} B^{r}.
\]
By the lower bound in Stirling's inequality,
\[
r! \ge \left(\frac{r}{e}\right)^{r},
\]
so that
\[
|R_r(x)| \le e^{B} \left(\frac{eB}{r}\right)^{r}.
\]
Thus, we wish to show that
\[
e^{B}\!\left(\frac{eB}{r}\right)^{r} \le \varepsilon
\quad\text{if }
r \ge 5\max\{B, \ln(1/\varepsilon)\}.
\]
We consider two cases.

\medskip
\noindent{Case 1:} $\ln(1/\varepsilon)\ge B$.  
Then $r\ge 5\ln(1/\varepsilon)$ and $r\ge 5B$, hence
\[
\frac{eB}{r}\le \frac{eB}{5\ln(1/\varepsilon)}\le \frac{e}{5}.
\]
Therefore, since $\frac{e}{5}<1$ and $B\le \ln(1/\varepsilon)$,
\[
e^{B}\!\left(\frac{eB}{r}\right)^{r}
\le \varepsilon^{-1}\!\left(\frac{e}{5}\right)^{r}
\le \varepsilon^{-1}\!\left(\frac{e}{5}\right)^{5\ln(1/\varepsilon)}
= \varepsilon^{ 5\ln 5-6} \leq \varepsilon.
\]
since $5\ln 5-6>1$.

\medskip
\noindent{Case 2:} $B\ge \ln(1/\varepsilon)$.  
Now $r\ge 5B$, so
\[
e^{B}\!\left(\frac{eB}{r}\right)^{r}
\le e^{B}\!\left(\frac{e}{5}\right)^{5B}
= \exp\!\big(B(6-5\ln 5)\big).
\]
Since $6-5\ln 5\approx -2.047\le -1$, we obtain
\[
\exp\!\big(B(6-5\ln 5)\big)\le e^{-B}\le e^{-\ln(1/\varepsilon)}=\varepsilon.
\]

\medskip
In both cases, $e^{B}(eB/r)^{r}\le\varepsilon$, thus we have
\[
r \ge 5\max\{B, \ln(1/\varepsilon)\}
\quad\Longrightarrow\quad
e^{B}\!\left(\frac{eB}{r}\right)^{r} \le \varepsilon.
\]

\section{Proof of Theorem \ref{theorem:theorem1}} \label{appendix:theorem1}
\begin{proof}
Consider a treatment unit \( n \in \mathcal{I}^{(1)} \) and time \( t \in [T] \). By Assumption~\ref{assump:factor_model} and ~\ref{assump:selection_latent},
\[
S_{1,n}^{(0)}(t) = \langle u_1(t), v_n \rangle.%
\]

By Assumption~\ref{assump:linearity_span}, there exists weights \( \{ w_{n,m} \}_{m \in \mathcal{I}^{(0)}} \) such that:
\[
v_n = \sum_{m \in \mathcal{I}^{(0)}} w_{n,m} \cdot v_m.
\]

Substituting into the expression above:
\[
S_{1,n}^{(0)}(t) = \left\langle u_1(t), \sum_{m \in \mathcal{I}^{(0)}} w_{n,m} v_m \right\rangle = \sum_{m \in \mathcal{I}^{(0)}} w_{n,m} \cdot \langle u_1(t), v_m \rangle=\sum_{m \in \mathcal{I}^{(0)}} w_{n,m} \cdot S_{1,m}^{(0)}(t).
\]

\end{proof}

\section{Time-to-event Estimation with Kaplan-Meier Estimator}
\label{sec:KM}

Fix a period-unit pair $(p,m)$ with $\widetilde K$ observed event times, i.e., $\sum_{i}\Delta_{p, m, i}=\widetilde K$. Let the distinct observed event times thus be
$
t_{(1)} < t_{(2)} < \cdots < t_{(\widetilde K)},
$
computed from $T_{p,m,i}$ with $\Delta_{p,m,i}=1$.
For each event time $t_{(j)}$, define:
\[
d_{p,m}(t_{(j)}) = \#\{ i : T_{p,m,i} = t_{(j)}, \ \Delta_{p,m,i} = 1\},
\]
\[
n_{p,m}(t_{(j)}) = \#\{ i: T_{p,m,i} \ge t_{(j)}\}.
\]

Then the KM estimator is
\begin{equation}
\widehat S^{d,K}_{p,m}(t)
=
\prod_{t_{(j)} \le t}
\left(1 - \frac{d_{p,m}(t_{(j)})}{n_{p,m}(t_{(j)})}\right),
\qquad t \in \mathcal{T}.
\label{eq:KM-pm}
\end{equation}

We adopt the convention $\widehat S^{d,K}_{p,m}(0)=1$.

\section{Finite Sample Guarantee of SSC}

Fix a period $p\in\{0,1\}$ and a finite evaluation grid 
$\mathcal T=\{t_1,\ldots,t_{T_0}\}\subseteq[0,\widetilde\tau]$.
For unit $m$ in period $p$, we denote the observed data to be $\{(D_{p, m, i}, T_{p,m,i},\Delta_{p,m,i})\}_{i=1}^{K}$ where $D_{p, m, i}$ is the assigned treatment, observed time $T_{p,m,i}=\min\{\tau_{p,m,i},C_{p,m,i}\}$, event indicator $\Delta_{p, m,i}=\mathbf 1\{\tau_{p,m,i}\le C_{p,m,i}\}$, $\tau_{p,m,i}$ survival times, 
$C_{p,m,i}$  censoring times satisfying the positivity condition $
\inf_{0\le t\le \widetilde\tau}\Pr(C_{p, m}>t) \ge \delta>0$, and $K$ the minimum number of observations per unit and period across all units and periods. $\tau_{p, m}$ and $C_{p,m}$ satisfy non-informative censoring (Assumption~\ref{ass:censoring}).

Let $\widehat S^{d, K}_{p,m}(\cdot)$ be the Kaplan--Meier (KM) estimate of the survival distribution %
based on
$\{(D_{p, m, i}, T_{p,m,i},\Delta_{p,m,i})\}_{i=1}^{K}$ if $D_{p, m, i}=d     \ \forall i$. Since the control units are under the control treatment across periods, $D_{p, m, i} = 0$ for $p \in \{0, 1\}$, for all $i$ and $m \in \mathcal{I}^{(0)}$. 

Define the KM estimation error matrix for all the control units $\mathcal{I}^{(0)}$ on $\mathcal T$ by
\begin{equation}
    E_p  =  \big[\widehat S^{0, K}_{p,m}(t)-S_{p, m}^{0}(t)\big]_{m\in \mathcal{I}^{(0)}, t\in\mathcal T} \label{eq:E_p}.
\end{equation}

All the following lemmas and proofs adopt the aforementioned notations and assumptions. All $O_p(\cdot)$ statements are with respect to $\min(T_0, N_0)$ where $T_0$ is the number of evaluation timestamps and $N_0$ is the number of control units. For simplicity, note that in the following sections we use $C$ to denote a generic constant whose value may change in different contexts but does not depend on problem parameters such as $N_0,T_0,r_0$.

Essentially, our proof of Theorem \ref{thm:consistency}  proceeds in three stages. We first establish uniform entry-wise bounds for $E_p$, the error in the Kaplan–Meier estimator in the presence of non-informative censoring (under appropriate conditions). 
This leads to operator–norm bounds for $E_p$. With large enough $K$, %
it will satisfy the desired stability property in noisy observations of low-rank covariates as required by Principal Component Regression (PCR) method. Towards that, as the next step, we state a generic PCR stability result analogous to what is known in literature (specifically, see \citet{agarwal2024syntheticinterventions}). Together, this would conclude that SSC manages to estimate the desired 
weights within small enough error. 
Finally, combining these with row–space inclusion property enables the method to extrapolate faithfully to estimate counterfactual survival function for the treated unit of interest. 

\subsection{Helper Lemmas for Theorem \ref{thm:consistency}}

In this section, we state and prove a few helper lemmas for invoking the generic PCR stability result following  \citet{agarwal2024syntheticinterventions}. Essentially, in the prior literatue \citet{agarwal2024syntheticinterventions}
the property of PCR is analyzed under low-rank covariates with their noisy observations where element-wise noise is
distributed as per $0$ mean, independent sub-Gaussian random variable. In our setting, the noise is induced due to estimation
error of Kaplan-Meier procedure, which is uniformly bounded but arbitrary. Towards that, we extend the PCR stability result
to account for such noisy setting.

\subsubsection{Perturbation of Singular Values}
First, we show the singular values of the true and estimated survival functions of the control units in the pre-period, $S^0_{0, \mathcal{I}^{(0)}}(\cdot)$ and $\widehat{S}^{0, K}_{0, \mathcal{I}^{(0)}}(\cdot)\in\mathbb R^{T_0\times N_0}$, are very close.

\begin{lemma}[Perturbation of Singular Values]\label{lem:KM-singular-values}
Conditioned on $\{\mathcal L\mathcal F,\mathcal D\}$, let $s_\ell$ and $\widehat s_\ell$ be the $\ell$-th singular values of the true and estimated survival functions of the control units in the pre-period, $S^0_{0, \mathcal{I}^{(0)}}(\cdot)$ and $\widehat{S}^{0, K}_{0, \mathcal{I}^{(0)}}(\cdot)\in\mathbb R^{T_0\times N_0}$ respectively. Then for any $\zeta>0$ and $\ell\le \min\{T_0,N_0\}$,
if
\begin{equation}
K  \ge  c \max\{N_0,T_0\} \bigl(\zeta^2+\log N_0+\log\max\{N_0,T_0\}\bigr) \label{eq:K-growth}
\end{equation}
then with probability at least $1-2e^{-\zeta^2}$,
\begin{equation} \label{eq:s_l}
     \bigl|s_\ell-\widehat s_\ell\bigr|
 \le  C\bigl(\sqrt{T_0}+\sqrt{N_0}\bigr),
\end{equation}
for some constants $c, C$.
\end{lemma}

\begin{proof}
    Given the operator norm bound in Lemma \ref{lem:KM-op}, applying the Weyl's inequality, we get \eqref{eq:s_l}.
\end{proof}

\begin{lemma}[Operator-norm control from survival estimation error]\label{lem:KM-op}
For a period $p$, let $E_p\in\mathbb R^{T_0\times N_0}$ defined in \eqref{eq:E_p}. 
Let the setup of Lemma \ref{lem:KM-singular-values} hold. Then for any $\zeta > 0$,
with probability at least  $1-2e^{-\zeta^2}$,
\begin{equation}\label{eq:op-target}
\|E_p\|_{op}  \le  C\big(\sqrt{T_0}+\sqrt{N_0}\big)
\end{equation} for some constant $C>0$. 
\end{lemma}

\begin{proof}
We use the following  uniform tail bound for the KM estimation (cf.\ \citet{foldesStrongUniformConsistency1981}):

\begin{lemma}[Uniform error for Kaplan--Meier]\label{lem:KM-uniform}
There exist absolute constants $A,B>0$ such that, for all $\varepsilon\in \big(12/(K\delta^4),1\big)$ where $\delta$ is the lower bound on the censoring survival probability ($\inf_{0\le t\le \widetilde\tau}\Pr(C_{p, m}>t) \ge \delta>0$),
\begin{equation}\label{eq:KM-tail}
\Pr\!\Big(\sup_{0\le t\le\widetilde\tau}\big|\widehat S(t)-S(t)\big|>\varepsilon\Big)
 \le \frac{A}{\varepsilon}\exp\!\big(-B K \varepsilon^2\big).
\end{equation}
\end{lemma}

Under our causal framework, the potential survival function for a unit under its realized treatment coincides with its true observational survival trajectory, which is estimated with the empirical Kaplan–Meier estimator. We hope to establish that the estimated empirical survival function is close to the true observational survival function and the potential survival function under the realized treatment as well. Given Lemma \ref{lem:KM-uniform}, let
\begin{equation}\label{eq:eps-choice}
\varepsilon^2  = 
\frac{c_0 \big(\log N_0 + \log\max\{N_0,T_0\} + \zeta^2\big)}{K},
\qquad
c_0 \ge \max\{2/B,2\}.
\end{equation}
For $K$ satisfying \eqref{eq:K-growth} and $C$ large enough, we have $\varepsilon \le 1$ and $\varepsilon \ge 12/(K\delta^4)$, so \eqref{eq:KM-tail} applies for all control units $m \in \mathcal{I}^{(0)}$ and period $p \in \{0, 1\}$.

By a union bound over $m \in \mathcal{I}^{(0)}$,
\begin{align}
&\Pr\!\Big(\max_{m \in \mathcal{I}^{(0)}}\sup_{0\le t\le\widetilde\tau}
\big|\widehat S^{0, K}_{p,m}(t)-S^0_{p, m}(t)\big|>\varepsilon\Big)\\
&\le
N_0 \frac{A}{\varepsilon} \exp(-BK\varepsilon^2) \nonumber\\
&=
\exp\!\Big(
\log N_0 + \log A - \log\varepsilon
- Bc_0\big(\log N_0 + \log\max\{N_0, T_0\} + \zeta^2\big)
\Big). \label{eq:union-pre}
\end{align}

Note that
\[
-\log\varepsilon
=
\frac{1}{2}\log K - \frac{1}{2}\log\!\big(c_0(\log N_0+\log\max\{N_0,T_0\}+\zeta^2)\big).
\]
From \eqref{eq:K-growth},
\[
\log K 
 \ge 
\log C + \log\max\{N_0,T_0\} 
+ \log\big(\zeta^2 + \log N_0 + \log\max\{N_0,T_0\}\big).
\]
Substituting this into \eqref{eq:union-pre} and canceling terms,
\begin{align*}
\log\Pr(\max_{m\in \mathcal{I}^{(0)},t\in \mathcal{T}}|\widehat S^{0, K}_{p,m}(t)-S^0_{p, m}(t)|>\varepsilon)
& \le 
- Bc_0 \zeta^2 
 -  (Bc_0-1)\log N_0 \\
&\qquad
 - (Bc_0 - \tfrac12)\log\max\{N_0,T_0\}.
\end{align*}
Since $c_0\ge 2/B$, we have $Bc_0-1\ge 1$ and $Bc_0 - \tfrac12 \ge \tfrac12$, hence for $C$ large enough,
\[
\Pr\!\Big(\max_{m\in \mathcal{I}^{(0)},t\in \mathcal{T}}|\widehat S^{0, K}_{p,m}(t)-S^0_{p, m}(t)|>\varepsilon\Big)  \le  2e^{-\zeta^2}.
\]
Thus, with probability at least $1-2e^{-\zeta^2}$,
\begin{equation}\label{eq:entrywise}
\max_{m\in \mathcal{I}^{(0)},t\in \mathcal{T}}|\widehat S^{0, K}_{p,m}(t)-S^0_{p, m}(t)|\leq \varepsilon.
\end{equation}

The \eqref{eq:entrywise} implies
\[
\|E_p\|_F^2 = \sum_{m \in \mathcal{I}^{(0)}}\sum_{t\in\mathcal T} (\widehat S^{0, K}_{p,m}(t)-S^0_{p, m}(t))^2
 \le  N_0 T_0 \varepsilon^2
\quad\Rightarrow\quad
\|E_p\|_{op}\le \|E_p\|_F \le \sqrt{N_0T_0} \varepsilon.
\]

Plugging \eqref{eq:eps-choice} and then \eqref{eq:K-growth},
\begin{align*}
\|E_p\|_{op}
& \le 
\sqrt{N_0T_0} 
\sqrt{\frac{c_0(\log N_0 + \log\max\{N_0,T_0\}+\zeta^2)}{K}}\\
& \le 
\sqrt{\frac{c_0 N_0T_0}{C \max\{N_0,T_0\}}}\\
& \le 
C\big(\sqrt{N_0} + \sqrt{T_0}\big),
\end{align*}
after adjusting the constant $C$. This proves \eqref{eq:op-target}. 
\end{proof}

\subsubsection{Rowspaces Inclusion} Second, we relate the rowspaces of the survival functions in different periods. Observe that Assumption \ref{ass:span} implies that there exist weights $\beta \in \mathbb{R}^{T_0}$ such that
\begin{equation}
    u_1(t) = \sum_{s \in \mathcal{T}} \beta_s u_0(s). \label{eq:KM-weights}
\end{equation}
\begin{lemma}[Rowspace Inclusion]  \label{lem:4}
Let $\beta$ be defined as in~\eqref{eq:KM-weights}. Then for any unit $m \in[N]$,
\begin{align}
    S_{1,m}(t)
 = 
\sum_{s \in \mathcal{T}} \beta_s \cdot
S_{0,m}(s).
\end{align}
\end{lemma}

\begin{proof}
We have
\begin{align*}
S_{1,m}(t)
&= \langle u_1(t), v_m \rangle \text{\qquad\qquad from Assumption \ref{assump:factor_model} }\\
&=
\big\langle \sum_{s \in \mathcal{T}} \beta_s 
u_{0}(s), v_m \big \rangle \text{\quad from \eqref{eq:KM-weights} }\\
&= \sum_{s \in \mathcal{T}}\big\langle  \beta_s 
u_{0}(s), v_m \big \rangle \\
&= \sum_{s\in \mathcal{T}} \beta_s \cdot S_{0, m}(s) \text{\qquad from Assumption \ref{assump:factor_model} }.
\end{align*}
This completes the proof.
\end{proof}

Going forward, to simply notations, we shall drop the $n$ subscript whenever evident since all of them are 
presented with respect to the target treated unit $n \in \mathcal{I}^{(1)}$.

\subsubsection{Representing Theorem \ref{theorem:theorem1} with the Minimum \texorpdfstring{$\ell_2$}{l2}-Norm Solution}

We present characterization of Theorem \ref{theorem:theorem1} by replacing any weight $w_m$ by 
the unique minimum-norm solution $\widetilde w_m$.
\begin{lemma}[Theorem \ref{theorem:theorem1} with minimum-norm solution weights]
\label{lem:km-lemma5}
Let $V_{0},V_{1}$ denote the row spaces of the true survival functions of the control units in the pre- and post-period,
$S^0_{0, \mathcal{I}^{(0)}}(\cdot)$ and $S^0_{1, \mathcal{I}^{(0)}}(\cdot)$, respectively, and define
\[
\widetilde w   :=  \mathcal P_{V_{0}} w
\quad\text{with}\quad
\mathcal P_{V_{0}}=V_{0}V_{0}^\top .
\]
Then
\[
\theta(t) = S_{1, n}^{(0)}(t)
 = 
\sum_{m\in\mathcal I^{(0)}} \widetilde w_m  S^0_{1,m}(t).
\]
\end{lemma}

\begin{proof}
By Lemma~\ref{lem:4}, the post-period row span is contained in the pre-period row span:
\[
\mathrm{rowspan}(S^0_{1,\mathcal{I}^{(0)}})
\ \subseteq\
\mathrm{rowspan}(S^0_{0,\mathcal{I}^{(0)}}).
\]
Hence, the orthogonal projector $\mathcal P_{V_0} := V_0 V_0^\top$ onto
$\mathrm{rowspan}(S^0_{0,\mathcal{I}^{(0)}})$ leaves $S^0_{1,\mathcal{I}^{(0)}}$ invariant:
\begin{equation}
S^0_{1,\mathcal{I}^{(0)}} \mathcal P_{V_0}
 = 
S^0_{1,\mathcal{I}^{(0)}}.
\label{eq:proj-invariance}
\end{equation}
Therefore,
\begin{align}
S^0_{1,\mathcal{I}^{(0)}} \widetilde w
&= S^0_{1,\mathcal{I}^{(0)}} \mathcal P_{V_0} w \nonumber\\
&= S^0_{1,\mathcal{I}^{(0)}} V_0 V_0^\top w 
\qquad\text{(by definition of $\mathcal P_{V_0}$)} \nonumber\\
&= S^0_{1,\mathcal{I}^{(0)}} w 
\qquad\qquad\text{(by \eqref{eq:proj-invariance})} \nonumber\\
&= \theta(t)\nonumber.
\end{align}
\end{proof}

\subsection{Stability of Principal Component Regression (PCR)}

We state a stability result for PCR building upon the prior works cf. \citet{agarwal2024syntheticinterventions}. It generalizes the prior work for the setting of arbitrary, bounded error.

\begin{lemma}[Stability of PCR]
\label{lem:generic-PCR}
Let $M= U\Sigma V^\top\in\mathbb{R}^{T_0\times N_0}$ be a rank-$r_0$ matrix with left singular space $U \in \mathbb{R}^{T_0 \times r_0}$, right singular space $V\in\mathbb{R}^{N_0\times r_0}$, and singular
values $s_1 \geq \cdots \geq  s_{r_0}  \geq \sqrt{\frac{T_0N_0}{r_0}}$. Assume $\|M\|_{\max}:=\max_{i \in [T_0], j \in [N_0]} | M_{i, j}|\leq 1$. Let $\mathcal P_{A}=AA^\top$ denote the projection matrix onto the subspace spanned by the columns of $A$ for any matrix $A$ with orthonormal columns.
Let $y \in \mathbb{R}^{T_0}$ be a target vector and let $\widetilde{w}\in\mathbb{R}^{N_0}$ 
denote the \emph{minimum $\ell_2$-norm} representation of $y$, i.e. %
\[
y  =  M \widetilde{w} + e,
\qquad
\widetilde{w}=\arg\min_{w: \|y-Mw\| \text{ is minimized}}\|w\|_2  .
\]
Let $e \in \mathbb{R}^{T_0}$ be such that \begin{equation} \label{eq:entrywise-conc-0}
    \|e\|_{\max} :=\max_{i \in [T_0]}| e_i| \le \frac{1}{\max\{\sqrt{T_0}, \sqrt{N_0}\}} 
\end{equation} 
and let $E\in\mathbb{R}^{T_0\times N_0}$ be such that %
\begin{align}
\label{eq:entrywise-conc-1}
&\|E\|_{\max}:=\max_{i \in [T_0], j \in [N_0]} |E_{i, j}| \le \frac{1}{\max\{\sqrt{T_0}, \sqrt{N_0}\}}.
\end{align} 
 Let the observed matrix 
\(Z = M + E\) with \( \widehat U \widehat\Sigma \widehat V^\top\) 
be its SVD, and its truncation using the top $r_0$ singular components is denoted as $Z^{r_0}=  \widehat U_{r_0} \widehat\Sigma_{r_0} \widehat V^\top_{r_0}.$ Let $\widehat{w}$ be the estimate obtained by PCR: obtain top $r_0$ singular vectors of $Z$ and project
$y$ onto the space spanned by them. Then
\begin{align}\label{eq:wtilde-l2l1}
\|\widetilde{w}\|_2
&\le
C \sqrt{\frac{r_0}{N_0}},
\qquad
\|\widetilde{w}\|_1 
\le C\sqrt{r_0},
\\
\label{eq:w2-rate}
\|\widehat{w}  - \widetilde{w}\|_2
&=
O\!\left(
\frac{r_0^{3/4}}{T_0^{1/4} N_0^{1/2}}
+
\frac{r_0^{3/2}}{\min\{T_0, N_0\}}
\right),
\\
\label{eq:proj-rate}
\|\mathcal{P}_{V_{0}}(\widehat{w} - \widetilde w)\|_2
&= 
O\!\left(
\frac{r_0^{2}}{\min\{T_0^{3/2}, N_0^{3/2}\}}
+
\frac{r_0^{3/2}}{N_0^{1/2}\min\{T_0^{1/2}, N_0^{1/2}\}} + \frac{r_0^{3/4}}{T_0^{1/4}N_0^{1/2}}
\right).
\end{align}
\end{lemma}
\begin{proof}
We follow the same calculation as the proof of Theorem 4.1 in \citet{agarwalModelIdentificationOutofSample2023} after matching the notations $\hat{\beta} = \hat w, \beta^* = \widetilde w, k = r_0, X=M, \widetilde Z = Z, \epsilon = e, \rho = 1, n = T_0, p=N_0$. Note that the calculation so far is purely algebraic and does not depend on the property of error terms, thus the expression can be directly translated in our notations that 
\begin{equation}
\|\widehat{w} - \widetilde{w}\|_2^2
 \le 
\| V V^\top - \widehat{V}_{r_0}\widehat{V}_{r_0}^\top \|_2^2  \|\widetilde{w}\|_2^2
 + 
\frac{4}{\widehat{s}_{r_0}^2}
\left(
\|M - Z^{r_0}\|_{2,\infty}^2  \|\widetilde{w}\|_1^2
 + 
+ \bigl|\langle Z^{r_0}(\widehat{w} - \widetilde{w}),  e \rangle\bigr|
\right).
\end{equation}

Similarly, by matching notations in Lemma 8 in \citet{agarwal2024syntheticinterventions}, we have 
$$
    \|\widetilde{w}\|_2
\le
C \sqrt{\frac{r_0}{N_0}},
\qquad
\|\widetilde{w}\|_1 
\le C\sqrt{r_0}.
$$
This proves \eqref{eq:wtilde-l2l1}.

From Lemma \ref{lem:KM-singular-values} that
$
|\hat s_r - s_r|
 \le 
C\bigl(\sqrt{T_0} + \sqrt{N_0}\bigr)
$
and the lower bound
$
s_{r_0}  \ge   \sqrt{\frac{T_0N_0}{r_0}},
$ we have $\hat s_{r_0} = \Theta(s_{r_0})$.

Next, to bound \eqref{eq:w2-rate} and \eqref{eq:proj-rate}, we state a few helper lemmas whose proofs can be found in Section \ref{appendix:proof_perturbation_singular_space}, \ref{appendix:proof_rowwise-pcr}, and \ref{appendix:inner-prod-error}.

\begin{lemma} \label{lem:perturbation_singular_space}
    Under the setup of Lemma~\ref{lem:generic-PCR}, 
\begin{equation}
\max\left(\bigl\| U U^\top - \widehat U_{r_0}\widehat U_{r_0}^\top \bigr\|_2, \bigl\| V V^\top - \widehat V_{r_0}\widehat V_{r_0}^\top \bigr\|_2\right)
 \le 
C\frac{\sqrt{T_0} + \sqrt{N_0}}{s_{r_0}} \label{eq:perturbation_singular_space}
\end{equation}  for some constant $C>0$.
\end{lemma}

\begin{lemma}\label{lem:rowwise-pcr}
Under the setup of Lemma~\ref{lem:generic-PCR}, 
\begin{equation}
\label{eq:rowwise_pcr}
\|M -  Z^{r_0}\|_{2,\infty}^2
 \le 
C
\left(\frac{T_0}{\max\{N_0,T_0\}}
 + 
\frac{T_0(T_0+N_0)}{s_{r_0}^2}\right)
,
\end{equation}
\end{lemma}

\begin{lemma}[]\label{lem:inner-prod-error}
Under the setup of Lemma~\ref{lem:generic-PCR},
\begin{equation}
\label{eq:inner-error-final}
\big|\langle Z^{r_0}(\widehat w - \widetilde w), e\rangle\big|
 \le 
C
\left(\|Z^{r_0}-M\|_{2,\infty} + \sqrt{T_0}\right)\|\widetilde w\|_1,
\end{equation}
for some absolute constant $C>0$.
\end{lemma}

Applying Lemmas \ref{lem:perturbation_singular_space}, \ref{lem:rowwise-pcr}, and \ref{lem:inner-prod-error}, we have
\begin{align*}
\|\widehat{w} - \widetilde{w}\|_2^2
&\leq
C\left(
\frac{T_0+N_0}{s_{r_0}^2}  \|\widetilde w\|_2^2
 + 
\frac{\|M - Z^{r_0}\|_{2,\infty}^2  \|\widetilde{w}\|_1^2}{\widehat{s}_{r_0}^2}
 + 
\frac{\sqrt{T_0} \|\widetilde w\|_1}{\widehat{s}_{r_0}^2}
\right) \\
&\leq
C\left(
\frac{T_0+N_0}{s_{r_0}^2} \cdot \frac{r_0}{N_0}
 + 
\frac{\|M - Z^{r_0}\|_{2,\infty}^2   r_0}{\widehat{s}_{r_0}^2}
 + 
\frac{\sqrt{T_0} \sqrt{r_0}}{\widehat{s}_{r_0}^2}
\right) \\
&\leq
C\left(
\frac{(T_0+N_0) r_0^2}{N_0^2 T_0}
 + 
\frac{r_0^2}{N_0 \max\{N_0,T_0\}}
 + 
\frac{r_0^3 (T_0+N_0)}{N_0^2 T_0}
 + 
\frac{r_0^{3/2}}{N_0 \sqrt{T_0}}
\right)\\
&\leq
C\left(
\frac{r_0^3 (T_0+N_0)}{N_0^2 T_0}
 + 
\frac{r_0^{3/2}}{N_0 \sqrt{T_0}}
\right)\\
&\leq
C\left(
\frac{r_0^3}{\min\{T_0, N_0\}^2}
 + 
\frac{r_0^{3/2}}{N_0 \sqrt{T_0}}
\right).
\end{align*}
Thus,
\begin{align*}
\|\widehat{w} - \widetilde{w}\|_2
&\leq
C\left(
\frac{r_0^{3/2}}{\min\{T_0, N_0\}}
 + 
\frac{r_0^{3/4}}{ T_0^{1/4}N_0^{1/2}}
\right).
\end{align*}
This proves \eqref{eq:w2-rate}.

To establish \eqref{eq:proj-rate}, we consider the following decomposition:
\begin{equation}
\label{eq:proj-decomp}
\mathcal{P}_V (\widehat w-\widetilde w)
=
\bigl(\mathcal{P}_V - \mathcal{P}_{\widehat V}\bigr)(\widehat w-\widetilde w)
 + 
\mathcal{P}_{\widehat V} (\widehat w-\widetilde w).
\end{equation}
We proceed to bound each term separately.

\paragraph{Bounding term 1.}
Recall that $\|A v\|_2 \le \|A\|_{\mathrm{op}} \|v\|_2$ for any
$A\in\mathbb{R}^{a\times b}$ and $v\in\mathbb{R}^b$.  Thus, from Lemma \ref{lem:perturbation_singular_space}
\begin{align}
\bigl\|
(\mathcal{P}_V - \mathcal{P}_{\widehat V})(\widehat w-\widetilde w)
\bigr\|_2
& \le 
\|\mathcal{P}_V - \mathcal{P}_{\widehat V}\|_{\mathrm{op}} \|(\widehat w-\widetilde w)\|_2\\
&\leq \frac{\sqrt{r_0}(\sqrt{T_0} + \sqrt{N_0})}{\sqrt{T_0N_0}} \left(
\frac{r_0^{3/2}}{\min\{T_0, N_0\}}
 + 
\frac{r_0^{3/4}}{ T_0^{1/4}N_0^{1/2}}
\right)\\
&\leq \frac{\sqrt{r}}{\min\{\sqrt{T_0}, \sqrt{N_0}\}} \left(
\frac{r_0^{3/2}}{\min\{T_0, N_0\}}
 + 
\frac{r_0^{3/4}}{ T_0^{1/4}N_0^{1/2}}
\right). \label{eq:term1-final}
\end{align}

\paragraph{Bounding term 2.}
Since $\widehat V_{r_0}$ is an isometry, it follows that
\begin{equation}
\label{eq:isometry}
\|\mathcal{P}_{\widehat V}(\widehat w-\widetilde w)\|_2^2
=
\|\widehat V_{r_0}^\top (\widehat w-\widetilde w)\|_2^2.
\end{equation}

We upper bound $\|\widehat V_{r_0}^\top(\widehat w-\widetilde w)\|_2^2$ as follows.  Consider
\[
\|Z^{r_0}(\widehat w-\widetilde w)\|_2^2
=
\bigl(\widehat V_{r_0}^\top (\widehat w-\widetilde w)\bigr)^\top
\widehat S_{r_0}^2
\bigl(\widehat V_{r_0}^\top (\widehat w-\widetilde w)\bigr)
 \ge 
\widehat s_{r_0}^2 \|\widehat V_{r_0}^\top(\widehat w-\widetilde w)\|_2^2.
\label{eq:S18-analog}
\]
Using \eqref{eq:isometry} and \eqref{eq:S18-analog} together implies
\begin{equation}
\label{eq:projVhat-delta}
\|\mathcal{P}_{\widehat V}(\widehat w-\widetilde w)\|_2^2
 \le 
\frac{\|Z^{r_0}(\widehat w-\widetilde w)\|_2^2}{\widehat s_{r_0}^2}.
\end{equation}

To bound the numerator in \eqref{eq:projVhat-delta}, we use the decomposition
\[
Z^{r_0}(\widehat w - \widetilde w)
=
\bigl(Z^{r_0}\widehat w - y\bigr)
 + 
\bigl(y - Z^{r_0}\widetilde w\bigr),
\qquad
y = M\widetilde w + e.
\]
Applying the inequality $\|x+y\|_2^2 \le 2\|x\|_2^2 + 2\|y\|_2^2$ gives
\begin{align}
\label{eq:Zr0Delta-split-correct}
\|Z^{r_0}(\widehat w - \widetilde w)\|_2^2
& \le 
2\|Z^{r_0}\widehat w - M\widetilde w\|_2^2
 + 
2\|(M - Z^{r_0})\widetilde w\|_2^2\\
& \le 
2\|Z^{r_0}\widehat w - M\widetilde w\|_2^2
 + 
2\|M - Z^{r_0}\|_{2,\infty}^2 \|\widetilde w\|_1^2. \label{eq:MminusZr0-wtilde}
\end{align}

Combining
\eqref{eq:projVhat-delta}
and \eqref{eq:MminusZr0-wtilde},
we obtain
\begin{equation}
\label{eq:projVhat-final-correct}
\|\mathcal{P}_{\widehat V}(\widehat w-\widetilde w)\|_2^2
 \le 
\frac{2}{\widehat s_{r_0}^2}
\left(
\|Z^{r_0}\widehat w - M\widetilde w\|_2^2
 + 
\|M - Z^{r_0}\|_{2,\infty}^2 \|\widetilde w\|_1^2
\right),
\end{equation}
which is the desired bound.

Next, we bound $\|Z^{r_0}\widehat w - M\widetilde w\|_2^2$. To this end, observe that
\begin{align}
\|Z^{r_0}\widehat w - y\|_2^2
&= \|Z^{r_0}\widehat w - M \widetilde w - e\|_2^2 \\
&= \|Z^{r_0}\widehat w - M \widetilde w \|_2^2 + \|e\|_2^2
 -  2\langle Z^{r_0}\widehat w - M \widetilde w ,  e\rangle.
\label{eq:S24prime}
\end{align}

By the optimality of $\widehat w$ (minimum–norm solution in the span of $\widehat V_{r_0}$), we have
\begin{align}
\|Z^{r_0}\widehat w - y\|_2^2
&= \|Z^{r_0}\widehat w - M\widetilde w - e\|_2^2  \\
&\le \|Z^{r_0}\widetilde w - M\widetilde w - e\|_2^2   \\
&= \|(Z^{r_0}-M)\widetilde w - e\|_2^2 \\
&= \|(Z^{r_0}-M)\widetilde w\|_2^2 + \|e\|_2^2
 -  2\langle (Z^{r_0}-M)\widetilde w,  e\rangle.
\label{eq:S25prime}
\end{align}

From \eqref{eq:S24prime} and \eqref{eq:S25prime}, we have
\begin{align}
\|Z^{r_0}\widehat w - M\widetilde w\|_2^2
&= \|Z^{r_0}\widehat w -y\|_2^2 - \|e\|_2^2 + 2\langle Z^{r_0}\widehat w - M \widetilde w ,  e\rangle\\
&\le \|(Z^{r_0}-M)\widetilde w\|_2^2
 +  2\langle Z^{r_0}(\widehat w - \widetilde w),  e\rangle \\
&\le \|Z^{r_0}-M\|_{2,\infty}^2 \|\widetilde w\|_1^2
 +  2\langle Z^{r_0}(\widehat w - \widetilde w),  e\rangle,
\label{eq:S26prime}
\end{align}
where the last inequality uses
\[
\|(Z^{r_0}-M)\widetilde w\|_2 \le \|Z^{r_0}-M\|_{2,\infty} \|\widetilde w\|_1.
\]

Thus, \begin{equation}
\|\mathcal{P}_{\widehat V}(\widehat w-\widetilde w)\|_2^2
 \le 
\frac{4}{\widehat s_{r_0}^2}
\left(
\|Z^{r_0}-M\|_{2,\infty}^2 \|\widetilde w\|_1^2
 +  + \bigl|\langle Z^{r_0}(\widehat{w} - \widetilde{w}),  e \rangle\bigr|
\right), \label{eq:term2-final}
\end{equation}

\paragraph{Collecting terms.}
Combining \eqref{eq:proj-decomp}, \eqref{eq:term1-final}, and
\eqref{eq:term2-final}, and applying Lemmas \ref{lem:rowwise-pcr} and \ref{lem:inner-prod-error} , we conclude that

\begin{align}
\|\mathcal{P}_V(\widehat w-\widetilde w)\|_2
&\leq 
\bigl\|(\mathcal{P}_V - \mathcal{P}_{\widehat V})(\widehat w-\widetilde w)\bigr\|_2
 + 
\bigl\|\mathcal{P}_{\widehat V}(\widehat w-\widetilde w)\bigr\|_2 \\
&\leq
\frac{\sqrt{r}}{\min\{\sqrt{T_0}, \sqrt{N_0}\}} \left(
\frac{r_0^{3/2}}{\min\{T_0, N_0\}}
 + 
\frac{r_0^{3/4}}{ T_0^{1/4}N_0^{1/2}}
\right)
\\
&+
\frac{2}{\widehat s_{r_0}}
\left(
\|Z^{r_0}-M\|_{2,\infty} \|\widetilde w\|_1
+ \bigl\langle Z^{r_0}(\widehat w - \widetilde w), e\bigr\rangle^{1/2}
\right) \\
&\leq
\frac{\sqrt{r}}{\min\{\sqrt{T_0}, \sqrt{N_0}\}} \left(
\frac{r_0^{3/2}}{\min\{T_0, N_0\}}
 + 
\frac{r_0^{3/4}}{ T_0^{1/4}N_0^{1/2}}
\right)\\
&+
\frac{\sqrt{r_0}}{\sqrt{T_0N_0}}
\Bigl(\|Z^{r_0}-M\|_{2,\infty}r_0^{1/2} + T_0^{1/4}r_0^{1/4}+\Bigr\|Z^{r_0}-M\|_{2,\infty}^{1/2}r_0^{1/4}\Bigl) \\
&\leq
\frac{\sqrt{r}}{\min\{\sqrt{T_0}, \sqrt{N_0}\}} \left(
\frac{r_0^{3/2}}{\min\{T_0, N_0\}}
 + 
\frac{r_0^{3/4}}{ T_0^{1/4}N_0^{1/2}}
\right) \\
&+
\frac{\sqrt{r_0}}{\sqrt{T_0N_0}}
\Bigl(\frac{\sqrt{r_0T_0}}{\min\{\sqrt{T_0}, \sqrt{N_0}\}}r_0^{1/2} + T_0^{1/4}r_0^{1/4}+\frac{r_0^{1/4}T_0^{1/4}}{\min\{T_0^{1/4}, N_0^{1/4}\}}r_0^{1/4}\Bigl)\\
&\leq
\frac{\sqrt{r}}{\min\{\sqrt{T_0}, \sqrt{N_0}\}} \left(
\frac{r_0^{3/2}}{\min\{T_0, N_0\}}
 + 
\frac{r_0^{3/4}}{ T_0^{1/4}N_0^{1/2}}
\right) \\
&+
\Bigl(\frac{r_0^{3/2}}{\sqrt{N_0}\min\{\sqrt{T_0}, \sqrt{N_0}\}} + \frac{r_0^{3/4}}{T_0^{1/4}N_0^{1/2}}+\frac{r_0}{T_0^{1/4}N_0^{1/2}\min\{T_0^{1/4}, N_0^{1/4}\}}\Bigl)\\
&\leq
\frac{r_0^{2}}{\min\{T_0^{3/2}, N_0^{3/2}\}}
+
\frac{r_0^{3/2}}{N_0^{1/2}\min\{T_0^{1/2}, N_0^{1/2}\}} + \frac{r_0^{3/4}}{T_0^{1/4}N_0^{1/2}}
\end{align}
This completes the proof.
\end{proof}

\subsection{Proof of Lemma \ref{lem:perturbation_singular_space}} \label{appendix:proof_perturbation_singular_space}

\begin{proof}
    First, we notice $\|E\|_{\mathrm{op}}\leq \|E\|_{F} \leq \sqrt{N_0T_0} \|E\|_{\max} \leq \frac{\sqrt{N_0T_0}}{\max\{\sqrt{T_0}, \sqrt{N_0}\}}\leq \sqrt{T_0} + \sqrt{N_0}.$
    From Wedin's Theorem \citep{wedinPerturbationBoundsConnection1972}, we know $$\max\left(\bigl\| U U^\top - \widehat U_{r_0}\widehat U_{r_0}^\top \bigr\|_2, \bigl\| V V^\top - \widehat V_{r_0}\widehat V_{r_0}^\top \bigr\|_2\right)
 \le 
\frac{2 \|E\|_{\mathrm{op}}}{s_{r_0}}.$$ Plugging in the upper bound of $\|E\|_\mathrm{op}$, we prove \eqref{eq:perturbation_singular_space}.
\end{proof}

\subsection{Proof of Lemma \ref{lem:rowwise-pcr}}\label{appendix:proof_rowwise-pcr}

\begin{proof}
We want to bound $\|M - Z^{r_0}\|_{2,\infty}^2$. To that end, for any
$j \in [N_0]$ define
\[
\Delta_j := M_{\cdot j} - Z^{r_0}_{\cdot j} \in \mathbb{R}^{T_0}.
\]
Our goal is to bound $\|\Delta_j\|_2^2$ uniformly over $j$, and then take
the maximum.

Recall that $Z^{r_0}$ is the rank-$r_0$ truncated SVD of $Z$, so
\[
Z^{r_0} = \widehat U_{r_0} \widehat\Sigma_{r_0} \widehat V_{r_0}^\top
\quad\text{and}\quad
Z^{r_0} = \widehat U_{r_0}\widehat U_{r_0}^\top Z.
\]
In particular,
\[
Z^{r_0}_{\cdot j} = \widehat U_{r_0}\widehat U_{r_0}^\top Z_{\cdot j}
\quad\text{for all } j\in[N_0].
\]
We decompose
\begin{align*}
    Z^{r_0}_{\cdot j} - M_{\cdot j}
    &= \bigl(Z^{r_0}_{\cdot j} - \widehat U_{r_0} \widehat U_{r_0}^\top M_{\cdot j}\bigr)
     + \bigl(\widehat U_{r_0} \widehat U_{r_0}^\top M_{\cdot j} - M_{\cdot j}\bigr).
\end{align*}
The first term lies in $\operatorname{span}(\widehat U_{r_0})$, while the
second term lies in its orthogonal complement:
\[
Z^{r_0}_{\cdot j} - \widehat U_{r_0} \widehat U_{r_0}^\top M_{\cdot j}
=
\widehat U_{r_0} \widehat U_{r_0}^\top (Z_{\cdot j} - M_{\cdot j})
\in \operatorname{span}(\widehat U_{r_0}),
\]
and, since $\widehat U_{r_0}\widehat U_{r_0}^\top$ is the orthogonal projector
onto $\operatorname{span}(\widehat U_{r_0})$,
\[
\widehat U_{r_0} \widehat U_{r_0}^\top M_{\cdot j} - M_{\cdot j}
= -\bigl(I - \widehat U_{r_0} \widehat U_{r_0}^\top\bigr) M_{\cdot j}
\in \operatorname{span}(\widehat U_{r_0})^\perp.
\]
Hence these two vectors are orthogonal, and we obtain
\begin{equation}
\label{eq:row-splitting}
\|Z^{r_0}_{\cdot j} - M_{\cdot j}\|_2^2
= \|Z^{r_0}_{\cdot j} - \widehat U_{r_0}\widehat U_{r_0}^\top M_{\cdot j}\|_2^2
+ \|\widehat U_{r_0}\widehat U_{r_0}^\top M_{\cdot j} - M_{\cdot j}\|_2^2.
\end{equation}

\paragraph{Bounding $\|Z^{r_0}_{\cdot j} - \widehat U_{r_0}\widehat U_{r_0}^\top M_{\cdot j}\|_2^2$.}
Using $Z^{r_0}_{\cdot j} = \widehat U_{r_0}\widehat U_{r_0}^\top Z_{\cdot j}$, we have
\begin{align*}
Z^{r_0}_{\cdot j} - \widehat U_{r_0}\widehat U_{r_0}^\top M_{\cdot j}
&= \widehat U_{r_0}\widehat U_{r_0}^\top Z_{\cdot j}
   - \widehat U_{r_0}\widehat U_{r_0}^\top M_{\cdot j} \\
&= \widehat U_{r_0}\widehat U_{r_0}^\top (Z_{\cdot j} - M_{\cdot j}).
\end{align*}
Therefore, by submultiplicativity of the operator norm and the fact that
$\widehat U_{r_0}\widehat U_{r_0}^\top$ is an orthogonal projector (so
$\|\widehat U_{r_0}\widehat U_{r_0}^\top\|_2 = 1$),
\begin{align}
\|Z^{r_0}_{\cdot j} - \widehat U_{r_0}\widehat U_{r_0}^\top M_{\cdot j}\|_2^2
&\le \|\widehat U_{r_0}\widehat U_{r_0}^\top\|_2^2  \|Z_{\cdot j} - M_{\cdot j}\|_2^2 \nonumber\\
&= \|Z_{\cdot j} - M_{\cdot j}\|_2^2. \label{eq:first-term-0}
\end{align}
Since $Z - M = E$ and, by assumption \eqref{eq:entrywise-conc-1},
\[
\|E\|_{\max} := \max_{i\in[T_0],\,j\in[N_0]} |E_{ij}|
\le \frac{1}{\max\{\sqrt{T_0},\sqrt{N_0}\}},
\]
every entry of $E_{\cdot j}$ has absolute value at most
$1/\max\{\sqrt{T_0},\sqrt{N_0}\}$. Hence
\begin{equation}
\label{eq:Zj-Mj-bound}
\|Z_{\cdot j} - M_{\cdot j}\|_2^2
= \|E_{\cdot j}\|_2^2
\le T_0 \cdot \frac{1}{\max\{T_0,N_0\}}
= \frac{T_0}{\max\{T_0,N_0\}}.
\end{equation}
Combining \eqref{eq:first-term-0} and \eqref{eq:Zj-Mj-bound}, we obtain
\begin{equation}
\label{eq:first-term}
\|Z^{r_0}_{\cdot j} - \widehat U_{r_0}\widehat U_{r_0}^\top M_{\cdot j}\|_2^2
\le \frac{T_0}{\max\{T_0,N_0\}}.
\end{equation}

\paragraph{Bounding $\|\widehat U_{r_0}\widehat U_{r_0}^\top M_{\cdot j} - M_{\cdot j}\|_2^2$.}
Recall that $M = U\Sigma V^\top$ with rank $r_0$, so $M_{\cdot j}\in\operatorname{span}(U)$ and thus
$UU^\top M_{\cdot j} = M_{\cdot j}$. Therefore
\begin{align}
\|\widehat U_{r_0}\widehat U_{r_0}^\top M_{\cdot j} - M_{\cdot j}\|_2^2
&= \|(\widehat U_{r_0}\widehat U_{r_0}^\top - UU^\top)M_{\cdot j}\|_2^2 \nonumber\\
&\le \|\widehat U_{r_0}\widehat U_{r_0}^\top - UU^\top\|_2^2
      \,\|M_{\cdot j}\|_2^2. \label{eq:second-term}
\end{align}
By the assumption $\|M\|_{\max} \le 1$, each entry of $M_{\cdot j}$ has
absolute value at most 1, so
\[
\|M_{\cdot j}\|_2^2 \le T_0.
\]
Moreover, Lemma~\ref{lem:perturbation_singular_space} implies
\[
\|\widehat U_{r_0}\widehat U_{r_0}^\top - UU^\top\|_2
\le C\frac{\sqrt{T_0}+\sqrt{N_0}}{s_{r_0}(M)}
\]
for some constant $C>0$, where $s_{r_0}(M)$ denotes the $r_0$-th singular
value of $M$. Substituting this bound and $\|M_{\cdot j}\|_2^2\le T_0$ into
\eqref{eq:second-term}, we obtain
\begin{equation}
\label{eq:second-term-final}
\|\widehat U_{r_0}\widehat U_{r_0}^\top M_{\cdot j} - M_{\cdot j}\|_2^2
\le C \frac{T_0(T_0+N_0)}{s_{r_0}^2}.
\end{equation}

\paragraph{Combining results.}
Substituting \eqref{eq:first-term} and \eqref{eq:second-term-final} into
\eqref{eq:row-splitting}, we find that for every $j\in[N_0]$,
\[
\|Z^{r_0}_{\cdot j} - M_{\cdot j}\|_2^2
\le \frac{T_0}{\max\{T_0,N_0\}}
 + C \frac{T_0(T_0+N_0)}{s_{r_0}^2}.
\]
Taking the maximum over $j\in[N_0]$ yields
\begin{align*}
    \|M - Z^{r_0}\|_{2,\infty}^2
    &= \max_{j\in[N_0]} \|M_{\cdot j} - Z^{r_0}_{\cdot j}\|_2^2 \\
    &\le
    C\left(\frac{ T_0}{\max\{N_0,T_0\}} + \frac{T_0(T_0+N_0)}{s_{r_0}^2}\right)
\end{align*}
for some constant $C>0$, as claimed.
\end{proof}

\subsection{Proof of Lemma \ref{lem:inner-prod-error}} \label{appendix:inner-prod-error}

\begin{proof}
Recall
\[
\widehat w
= \widehat V_{r_0}\widehat\Sigma_{r_0}^{-1}\widehat U_{r_0}^\top y,
\qquad
Z^{r_0}
= \widehat U_{r_0}\widehat\Sigma_{r_0}\widehat V_{r_0}^\top,
\qquad
y = M\widetilde w + e.
\]
Thus,
\[
Z^{r_0}\widehat w
= \widehat U_{r_0}\widehat U_{r_0}^\top y
= \widehat U_{r_0}\widehat U_{r_0}^\top M\widetilde w
  + \widehat U_{r_0}\widehat U_{r_0}^\top e,
\]
and
\[
Z^{r_0}\widetilde w
= \widehat U_{r_0}\widehat\Sigma_{r_0}\widehat V_{r_0}^\top\widetilde w.
\]
Hence,
\begin{equation}
\label{eq:inner-decomp}
\langle Z^{r_0}(\widehat w - \widetilde w), e\rangle
=
\underbrace{\langle \widehat U_{r_0}\widehat U_{r_0}^\top M\widetilde w,  e\rangle}_{(I)}
 + 
\underbrace{\langle \widehat U_{r_0}\widehat U_{r_0}^\top e,  e\rangle}_{(II)}
 - 
\underbrace{\langle \widehat U_{r_0}\widehat\Sigma_{r_0}\widehat V_{r_0}^\top\widetilde w,  e\rangle}_{(III)}.
\end{equation}
\paragraph{Bounding $(I)$}
By Cauchy--Schwarz, $\|e\|_2 \leq \sqrt{T_0}\frac{1}{\max\{\sqrt{T_0}, \sqrt{N_0}\}}, \|M\|_{2,\infty}\le \sqrt{T_0}$ ,
\[
|(I)|
\le \|M\widetilde w\|_2 \|e\|_2
\le \|M\|_{2,\infty}\|\widetilde w\|_1 \cdot \sqrt{\tfrac{T_0}{ \max\{T_0, N_0\}}}
\le \frac{T_0}{\sqrt{ \max\{T_0, N_0\}}} \|\widetilde w\|_1 \leq \sqrt{T_0}\|\widetilde w\|_1.
\]

\paragraph{Bounding $(II)$}
Since $\widehat U_{r_0}\widehat U_{r_0}^\top$ is an orthogonal projector,
\[
|(II)| \le \|e\|_2^2
\le T_0 \|e\|_\infty^2
\le \frac{T_0}{ \max\{T_0, N_0\}} \leq 1.
\]

\paragraph{Bounding $(III)$}
Write
\[
\widehat U_{r_0}\widehat\Sigma_{r_0}\widehat V_{r_0}^\top\widetilde w
= Z^{r_0}\widetilde w
= (Z^{r_0}-M)\widetilde w + M\widetilde w.
\]
Thus,
\[
\|\widehat U_{r_0}\widehat\Sigma_{r_0}\widehat V_{r_0}^\top\widetilde w\|_2
\le \left(\|Z^{r_0}-M\|_{2,\infty} + \sqrt{T_0}\right)\|\widetilde w\|_1.
\]
Using $\|e\|_2\le \sqrt{T_0/ \max\{T_0, N_0\}}$,
\[
|(III)|
\le \frac{\sqrt{T_0}}{\sqrt{ \max\{T_0, N_0\}}}
\left(\|Z^{r_0}-M\|_{2,\infty} + \sqrt{T_0}\right)\|\widetilde w\|_1\le 
\left(\|Z^{r_0}-M\|_{2,\infty} + \sqrt{T_0}\right)\|\widetilde w\|_1.
\]

\paragraph{Combining bounds.}
From \eqref{eq:inner-decomp}, the triangle inequality and Lemma \ref{lem:rowwise-pcr},
\begin{align}
\big|\langle Z^{r_0}(\widehat w - \widetilde w), e\rangle\big|
&\le |(I)| + |(II)| + |(III)| \nonumber\\
&\le
C
\left(\|Z^{r_0}-M\|_{2,\infty} + \sqrt{T_0}\right)\|\widetilde w\|_1
,
\end{align}
for some absolute constant $C>0$.
\end{proof}

\subsection{Proof of Theorem \ref{thm:consistency}} \label{sec:proof_thm2}

With our helper lemmas in the last subsection, we now formally establish Theorem \ref{thm:consistency}. For any matrix $A$ with orthonormal columns, let $\mathcal{P}_A = A A^\top$ 
denote the projection matrix onto the subspace spanned by the columns of $A$.

We first state how we use Lemma \ref{lem:generic-PCR}. In our survival setting, when $K  \ge  c \max\{N_0,T_0\} \bigl(\zeta^2+\log N_0+\log\max\{N_0,T_0\}\bigr)$ for some constant $c, \zeta>0$, we take
\[
M  =  S^{0}_{0,\mathcal{I}^{(0)}}(\cdot)\in\mathbb{R}^{T_0\times N_0}, 
\qquad
Z  =  \widehat{S}^{0,K}_{0,\mathcal{I}^{(0)}}(\cdot) = M+E,
\qquad
E  =  E_0,
\]
where $E_0$ is the KM error matrix on the pre-period grid (cf. \eqref{eq:E_p} with $p=0$) and the smallest non-zero singular value of $M$ satisfies, based on Assumption \ref{ass:spectrum}, $s_{r_0} \geq \sqrt{\frac{N_0T_0}{r}}$. The target vector at each $t\in\mathcal T$ is
\[
y  =  \widehat S^{0}_{1,\mathcal{I}^{(0)}}(t)\in\mathbb{R}^{N_0}.
\]
The concentration condition \eqref{eq:entrywise-conc-0} and \eqref{eq:entrywise-conc-1} holds w.h.p.\ by the KM uniform error bound and its operator-norm corollary. Substituting these identifications into the lemma yields \eqref{eq:w2-rate}–\eqref{eq:proj-rate} for the SSC weights. 

\paragraph{Decomposition.} 
For $\forall t \in \mathcal{T}$, by Lemma \ref{lem:km-lemma5}, we have
\begin{align}
    \widehat{\theta}(t) - \theta(t) &= \left\langle \widehat{S}_{1, \mathcal{I}^{(0)}}^0(t), \widehat{w}\right\rangle - \left\langle S_{1, \mathcal{I}^{(0)}}^0(t), \widetilde{w} \right\rangle \\
    &=\underbrace{\left\langle S_{1, \mathcal{I}^{(0)}}^0(t) , \widehat{w} - \widetilde{w}\right\rangle}_\text{(I)}+ \underbrace{\left\langle \widehat{S}_{1, \mathcal{I}^{(0)}}^0(t) - S_{1, \mathcal{I}^{(0)}}^0(t), \widetilde{w} \right\rangle}_\text{(II)}\\
    & + \underbrace{\left\langle \widehat{S}_{1, \mathcal{I}^{(0)}}^0(t) - S_{1, \mathcal{I}^{(0)}}^0(t), \widehat{w}-\widetilde{w} \right\rangle}_\text{(III)}\\
    &=\underbrace{\left\langle S^{0}_{1,\mathcal{I}^{(0)}}(t) , \widehat{w} - \widetilde w\right\rangle}_\text{(I)}+ \underbrace{\left\langle e_{1, \mathcal{I}^{(0)}, t}, \widetilde{w} \right\rangle}_\text{(II)} + \underbrace{\left\langle e_{1, \mathcal{I}^{(0)}, t}, \widehat{w}-\widetilde{w} \right\rangle}_\text{(III)} \label{eq:theta-KM-full}
\end{align}
where $e_{p, \mathcal{I}^{(0)}, t}=\left [\widehat S^{0, K}_{p,m}(t)-S^0_{p, m}(t)\right ]_{m \in \mathcal{I}^{(0)}}$.

We now bound the three terms in \eqref{eq:theta-KM-full} separately.

\paragraph{Bounding term 1.}

From \eqref{eq:proj-invariance}, $S^0_{1,\mathcal{I}^{(0)}} \mathcal P_{V_0}
 = 
S^0_{1,\mathcal{I}^{(0)}}.$
Thus, 
\begin{equation}
\left\langle S^{0}_{1,\mathcal{I}^{(0)}}(t) , \widehat{w} - w\right\rangle
= 
\left\langle S^{0}_{1,\mathcal{I}^{(0)}}(t) , \mathcal{P}_{V_{0}}(\widehat{w} - w)\right\rangle.
\label{eq:proj-pre}
\end{equation}

By Cauchy--Schwarz, 
\[
\left\langle S^{0}_{1,\mathcal{I}^{(0)}}(t) , \mathcal{P}_{V_{0}}(\widehat{w} - \widetilde w)\right\rangle
\le
\|S^{0}_{1,\mathcal{I}^{(0)}}(t)\|_2 \cdot \|\mathcal{P}_{V_{0}}(\widehat{w} - w)\|_2.
\]
Since survival functions are naturally bounded between 0 and 1, $\|S^{0}_{1,\mathcal{I}^{(0)}}(t)\|_2 \lesssim \sqrt{N_0}$. Thus it remains to bound $\|\mathcal{P}_{V_{0}}(\widehat{w} - w)\|_2$ for which we use \eqref{eq:proj-rate} in Lemma \ref{lem:generic-PCR}.

Thus,
\[
\big\langle S^{0}_{1,\mathcal{I}^{(0)}}(t), \mathcal{P}_{V_{0}}(\widehat{w} - \widetilde w)
\big\rangle
= O_p\!\left(\frac{N_0^{1/2}r_0^{2}}{\min\{T_0^{3/2}, N_0^{3/2}\}}
+
\frac{r_0^{3/2}}{\min\{T_0^{1/2}, N_0^{1/2}\}} + \frac{r_0^{3/4}}{T_0^{1/4}}
\right).
\]

\paragraph{Bounding term 2.} Since the estimation errors are independent across units and the true minimum $\ell_2$ norm solution $\widetilde{w}$ depends only on the true survival data in period 0, not on the period 1 estimation error, thus
$\left\langle e_{1, \mathcal{I}^{(0)}, t}, \widetilde{w} \right\rangle
=O_p( \|\widetilde{w}\|_2)=O_p(\sqrt{\frac{r_0}{N_0}})$ by \eqref{eq:wtilde-l2l1} in Lemma \ref{lem:generic-PCR}.

\paragraph{Bounding term 3.}
Define
\[
\mathcal{E}_1 = \left\{ \|\widehat{w}-\widetilde{w}\|_2 = 
O\!\left(
\frac{r_{0}^{3/4}}{T_0^{1/4} N_0^{1/2}}
 + 
\frac{r_{0}^{3/2}}{\min\{T_0, N_0\}}
\right) \right\}.
\]
By \eqref{eq:w2-rate} in Lemma \ref{lem:generic-PCR}:, $\mathcal{E}_1$ holds w.h.p.
Define
\[
\mathcal{E}_2 = \left\{ \left\langle e_{1, \mathcal{I}^{(0)}, t}, \widehat{w}-\widetilde{w} \right\rangle 
=
O\!\left(
\frac{r_{0}^{3/4}}{T_0^{1/4} N_0^{1/2}}
 + 
\frac{r_{0}^{3/2}}{\min\{T_0, N_0\}}
\right)\right\}.
\]

Again, since the estimation errors are independent across units and the learned weights $\widehat{w}$ are independent of estimation errors in period 1, by \eqref{eq:w2-rate} in Lemma \ref{lem:generic-PCR}, $\mathcal{E}_2$ holds w.h.p.

Thus,
\[
\left\langle e_{1, \mathcal{I}^{(0)}, t}, \widehat{w}-\widetilde{w} \right\rangle
=
O_p\!\left(
\frac{r_{0}^{3/4}}{T_0^{1/4} N_0^{1/2}}
 + 
\frac{r_{0}^{3/2}}{\min\{T_0, N_0\}}
\right).
\]

\paragraph{Collecting terms.}
Plugging all three bounds back to \eqref{eq:theta-KM-full} gives us the ideal result. \qed

\end{document}